\documentclass[onefignum, onetabnum]{siamonline190516}

\usepackage[utf8]{inputenc} %
\usepackage[T1]{fontenc}    %
\usepackage{amsfonts}       %
\usepackage{nicefrac}       %
\usepackage{enumitem}
\usepackage{pifont}
\usepackage{wrapfig}
\usepackage{xr}

\newsiamthm{assumption}{Assumption}

\usepackage{subcaption}     %
\usepackage{multirow}
\usepackage{tabularx}
\usepackage{booktabs}

\usepackage{footnote}
\usepackage{chngpage}
\usepackage{placeins}
\usepackage{ragged2e}
\usepackage[numbers]{natbib}
\makeatletter
\newcommand*{\addFileDependency}[1]{%
  \typeout{(#1)}
  \@addtofilelist{#1}
  \IfFileExists{#1}{}{\typeout{No file #1.}}
}
\makeatother

\usepackage[capitalize]{cleveref}
\Crefname{section}{Sec.}{Secs.}
\Crefname{section}{Section}{Sections}
\Crefname{table}{Table}{Tables}
\Crefname{table}{Tab.}{Tabs.}

\crefformat{equation}{\textup{#2(#1)#3}}
\crefrangeformat{equation}{\textup{#3(#1)#4--#5(#2)#6}}
\crefmultiformat{equation}{\textup{#2(#1)#3}}{ and \textup{#2(#1)#3}}
{, \textup{#2(#1)#3}}{, and \textup{#2(#1)#3}}
\crefrangemultiformat{equation}{\textup{#3(#1)#4--#5(#2)#6}}%
{ and \textup{#3(#1)#4--#5(#2)#6}}{, \textup{#3(#1)#4--#5(#2)#6}}{, and \textup{#3(#1)#4--#5(#2)#6}}

\Crefformat{equation}{#2Equation~\textup{(#1)}#3}
\Crefrangeformat{equation}{Equations~\textup{#3(#1)#4--#5(#2)#6}}
\Crefmultiformat{equation}{Equations~\textup{#2(#1)#3}}{ and \textup{#2(#1)#3}}
{, \textup{#2(#1)#3}}{, and \textup{#2(#1)#3}}
\Crefrangemultiformat{equation}{Equations~\textup{#3(#1)#4--#5(#2)#6}}%
{ and \textup{#3(#1)#4--#5(#2)#6}}{, \textup{#3(#1)#4--#5(#2)#6}}{, and \textup{#3(#1)#4--#5(#2)#6}}

\crefdefaultlabelformat{#2\textup{#1}#3}

\usepackage{enumitem}
\setlist[enumerate]{leftmargin=.5in}
\setlist[itemize]{leftmargin=.5in}

\newsiamremark{remark}{Remark}
\newsiamremark{hypothesis}{Hypothesis}
\crefname{hypothesis}{Hypothesis}{Hypotheses}
\newsiamthm{claim}{Claim}

\headers{Improving the Accuracy-Robustness Trade-Off with Adaptive Smoothing}{Y.\ Bai, B.\ G.\ Anderson, A.\ Kim, and S.\ Sojoudi}

\title{Improving the Accuracy-Robustness\\Trade-Off of Classifiers via Adaptive Smoothing\thanks{This work is an extension of \citep{Bai23a}.\funding{This work was supported by grants from ONR, NSF, and C3 AI Digital Transformation Institute.}}}

\author{Yatong Bai\thanks{Department of Mechanical Engineering and Department of Electrical Engineering and Computer Sciences, University of California, Berkeley, (\email{yatong\_bai@berkeley.edu}, \email{bganderson@berkeley.edu}, \email{sojoudi@berkeley.edu}).}
\and Brendon G.\ Anderson$^\dagger \hspace{-1.5mm}$
\and Aerin Kim\thanks{Scale AI, (\email{aerinykim@gmail.com}).}
\and Somayeh Sojoudi$^\dagger$
}

\ifpdf
\hypersetup{
  pdftitle = {Improving the Accuracy-Robustness Trade-Off of Classifiers via Adaptive Smoothing},
  pdfauthor= {Y.\ Bai, B.\ G.\ Anderson, A.\ Kim, and S.\ Sojoudi}
}
\fi

\usepackage[utf8]{inputenc}
\usepackage{amsfonts,amsmath,amssymb} %
\usepackage{braket}
\usepackage{mathtools,xparse}
\usepackage{bbm}
\usepackage{setspace}

\def\eqref#1{equation~\ref{#1}}

\def\1{\bm{1}}

\DeclareMathAlphabet{\mathsfit}{\encodingdefault}{\sfdefault}{m}{sl}
\SetMathAlphabet{\mathsfit}{bold}{\encodingdefault}{\sfdefault}{bx}{n}

\def\gD{{\mathcal{D}}}

\def\gF{{\mathcal{F}}}

\def\gN{{\mathcal{N}}}

\def\gS{{\mathcal{S}}}

\def\sI{{\mathbb{I}}}

\def\sP{{\mathbb{P}}}

\def\sR{{\mathbb{R}}}

\newcommand{\E}{\mathbb{E}}

\DeclareMathOperator{\sgn}{sgn}

\DeclareMathOperator*{\argmax}{arg\,max}

\DeclareMathOperator{\sign}{sign}

\DeclarePairedDelimiter\norm{\lVert}{\rVert}
\DeclarePairedDelimiter\bignorm{\big\lVert}{\big\rVert}

\DeclarePairedDelimiterX{\inp}[2]{\big\langle}{\big\rangle}{#1, #2}

\newcommand{\hismoa}{h_{\text{smo1}, i}^{\gamma}}
\newcommand{\hismob}{h_{\text{smo2}, i}^{\gamma}}
\newcommand{\hismoc}{h_{\text{smo3}, i}^{\gamma}}

\newcommand{\hyalpha}{h_y^{\alpha}}
\newcommand{\hialpha}{h_i^{\alpha}}

\newcommand{\hbase}{h_\mathrm{baseline}}
\newcommand{\halpha}{h^{\alpha}}
\newcommand{\halphat}{h^{\alpha_t}}
\newcommand{\halphatheta}{h^{\alpha_{\theta}}}

\renewcommand{\th}{^\text{th}}

\DeclareMathOperator{\lip}{Lip}
\newcommand{\rlippalpha}{r_{\mathrm{Lip},p}^\alpha}

\newcommand{\rsigmaalpha}{r_\sigma^\alpha}
\newcommand{\onemualpha}{1+\mu_\alpha}

\newcommand{\lBCE}{\ell_{\mathrm{BCE}}}
\newcommand{\lCE}{\ell_{\mathrm{CE}}}
\newcommand{\lsurr}{\ell_{\mathrm{surrogate}}}
\newcommand{\lcomp}{\ell_{\mathrm{composite}}}
\newcommand{\cBCE}{c_{\mathrm{BCE}}}
\newcommand{\cCE}{c_{\mathrm{CE}}}
\newcommand{\cprod}{c_{\mathrm{prod}}}

\newcommand{\alphamin}{\alpha_{\mathrm{min}}}
\newcommand{\alphamax}{\alpha_{\mathrm{max}}}

\newcommand{\cmark}{\ding{51}}
\newcommand{\xmark}{\ding{55}}

\newcommand{\about}{\textrm{\raisebox{0.5ex}\texttildelow}}

\begin{document}

\maketitle

\begin{abstract}
While prior research has proposed a plethora of methods that build neural classifiers robust against adversarial robustness, practitioners are still reluctant to adopt them due to their unacceptably severe clean accuracy penalties.
Real-world services based on neural networks are thus still unsafe.
This paper significantly alleviates the accuracy-robustness trade-off by mixing the output probabilities of a standard classifier and a robust classifier, where the standard network is optimized for clean accuracy and is not robust in general.
We show that the robust base classifier's confidence difference for correct and incorrect examples is the key to this improvement.
In addition to providing empirical evidence, we theoretically certify the robustness of the mixed classifier under realistic assumptions. 
We then adapt an adversarial input detector into a mixing network that adaptively adjusts the mixture of the two base models, further reducing the accuracy penalty of achieving robustness.
The proposed flexible mixture-of-experts framework, termed ``adaptive smoothing'', works in conjunction with existing or even future methods that improve clean accuracy, robustness, or adversary detection.
We use strong attack methods, including AutoAttack and adaptive attacks, to evaluate our models' robustness.
On the CIFAR-100 dataset, we achieve an $85.21 \%$ clean accuracy while maintaining a $38.72 \%$ $\ell_\infty$-AutoAttacked ($\epsilon = \nicefrac{8}{255}$) accuracy, becoming the second most robust method on the RobustBench benchmark as of submission, while improving the clean accuracy by ten percentage points over all listed models.
Code implementation is available at \url{https://github.com/Bai-YT/AdaptiveSmoothing}.
\end{abstract}



\section{Introduction} \label{sec:intro}

Neural networks are vulnerable to adversarial attacks in various applications, including computer vision and audio \citep{Moosavi-Dezfooli16, Goodfellow15}, natural language processing \citep{Fursov22}, and control systems \citep{Huang17}. Due to the widespread application of neural classifiers, ensuring their reliability in practice is paramount.

To mitigate this susceptibility, researchers have explored ``adversarial training'' (AT) and its improved variants \citep{Kurakin17, Goodfellow15, Bai22a, Bai22b, Zheng20}, building empirically robust models by training with adversarial examples.
Meanwhile, theoretical research has also considered certifying (i.e., mathematically guaranteeing) the robustness of neural classifiers against adversarial perturbations within a radius \citep{Anderson20, Ma20, Anderson21a}.
``Randomized smoothing'' (RS) is one such method that achieves certified robustness with an already-trained model at inference time \citep{Cohen19, Li19}.
Improved variants of RS incorporate dimension reduction methods \citep{Pfrommer23} and denoising modules \citep{Carlini22}. 
Recent work \citep{Anderson21b} has demonstrated that a data-driven locally biased smoothing approach can improve over traditional data-blind RS.
However, this method is limited to the binary classification setting and suffers from the performance bottleneck of its underlying one-nearest-neighbor classifier.

Despite the emergence of these proposed remedies to the adversarial robustness issue, many practitioners are reluctant to adopt them.
As a result, existing publicly available services are still vulnerable \citep{Ilyas18, Borkar21}, presenting severe safety risks. One important reason for this reluctance is the potential for significantly reduced model performance on clean data.
Specifically, some previous works have suggested a fundamental trade-off between accuracy and robustness \citep{Tsipas19, Zhang19}.
Since the sacrifice in unattacked performance is understandably unacceptable in real-world scenarios, developing robust classifiers with minimal clean accuracy degradation is crucial.

Fortunately, recent research has argued that it should be possible to simultaneously achieve robustness and accuracy on benchmark datasets \citep{Yang20}.
To this end, variants of adversarial training that improve the accuracy-robustness trade-off have been proposed, including TRADES \citep{Zhang19}, Interpolated Adversarial Training \citep{Lamb19}, Instance Adaptive Adversarial Training (IAAT) \citep{Balaji19}, and many others \citep{Cheng22, Chen20a, TBai21, Raghunathan20, Wang19a, Wang19b, Tramer18}.
However, despite these improvements, degraded clean accuracy is often an inevitable price of achieving robustness.
Moreover, standard non-robust models often achieve enormous performance gains by pre-training on larger datasets with self- or semi-supervision \citep{He22, LGS}. In contrast, the effect of pre-training on robust classifiers is less understood and may be less prominent \citep{Chen20b, Fan21}.
As a result, the performance gap between these existing works and the possibility guaranteed in \citep{Yang20} is still huge.

This work builds upon locally biased smoothing \citep{Anderson21b} and makes a theoretically disciplined step towards reconciling adversarial robustness and clean accuracy, significantly closing this performance gap and thereby providing practitioners additional incentives for deploying robust models.
This paper is organized as follows.
\begin{itemize}[leftmargin=8mm]
	\setlength\itemsep{1pt}
	\item In \Cref{sec:STD+ROB}, observing that the $K$-nearest-neighbor ($K$-NN) classifier, a crucial component of locally biased smoothing, becomes a performance bottleneck, we replace it with a robust neural network that can be obtained via various existing methods, and propose a new smoothing formulation accordingly.
	The resulting formulation \cref{eq:adap_sm_4} is a convex combination of the output probabilities of a standard neural network and a robust one.
	When the robust neural network has a certified Lipschitz constant or is based on randomized smoothing, the mixed classifier also has a certified robust radius.
	These contents are presented in our conference submission \citep{Bai23a}, but are strengthened in this paper.
	\item In \Cref{sec:ada_smo}, we propose adaptive smoothing, which adaptively adjusts the mixture of a standard model and a robust model by adopting a type of adversary detector as a ``mixing network''.
	The mixing network controls the convex combination of the output probabilities from the two base networks, further improving the accuracy-robustness trade-off, making the resulting model a mixture-of-experts design.
	We empirically verify the robustness of the proposed method using gray-box and white-box projected gradient descent (PGD) attack, AutoAttack, and adaptive attacks, demonstrating that the mixing network is robust against the attack types it is trained with.
	When the mixing network is trained with a carefully designed adaptive AutoAttack, the composite model significantly gains clean accuracy while sacrificing little robustness.
	This section and the corresponding experiment results are entirely new relative to our conference submission \citep{Bai23a}, and are crucial for achieving the much improved accuracy-robustness trade-off over existing works.
\end{itemize}
\vspace{.8mm}

Compared to existing methods for improving the accuracy-robustness trade-off, most of which are training-based, adaptive smoothing has several key advantages:
\begin{itemize}[leftmargin=8mm]
	\setlength\itemsep{1pt}
	\item Adaptive smoothing is agnostic to how the standard and robust base models are trained.
	Hence, one can quickly swap the base classifiers with already-trained standard or robust models.
	Therefore, our method is highly versatile and can be coupled with existing training-based trade-off improving methods.
	\item Adaptive smoothing can thus take advantage of pre-training on large datasets via the standard base classifier and benefit from ongoing advancements in robust training methods via the robust base model.
	Meanwhile, training-based methods have limited compatibilities, since they may conflict with certain techniques essential to achieving state-of-the-art (SOTA) clean or robust accuracy.
	As a result, adaptive smoothing achieves better results: it significantly boosts clean accuracy while maintaining near-SOTA robustness.
	\item Adaptive smoothing allows for an interpretable continuous adjustment between accuracy and robustness at inference time, which can be achieved by simply adjusting the mixture ratio.
	On the other hand, not all training-based methods allow for this adjustment. For those that do, this adjustment involves training an entirely new robust model.
	\item When the mixing ratio is fixed and the robust base model has a certified robust radius with a nonzero margin, the mixed classifier can be certified.
	Since certified models are often also certifiable with a nonzero margin, this condition is commonly satisfied in practice.
	For empirically robust base classifiers that are not certifiable, an estimation can be performed.
\end{itemize}

During the reviewing period of this paper, the authors of \citep{Li23} verified that our mixed classifier simultaneously improves the clean accuracy and the robustness against out-of-distribution (OOD) adversarial attacks (i.e., the threat model differs between training and evaluation), achieving state-of-the-art OOD adversarial robustness among a plethora of models, including the robust base classifier of our mixed classifier. This observation further bolsters the thesis that our proposed method achieves the accuracy-robustness trade-off.

\section{Background and Related Works}

\subsection{Notations} \label{sec:notations}

The symbol $\norm{\cdot}_p$ denotes the $\ell_p$ norm of a vector and $\norm{\cdot}_{p*}$ denotes its dual norm. For a scalar $a$, $\sgn (a) \in \{ -1, 0, 1 \}$ denotes its sign. For a natural number $c$, $[c]$ represents $\{1, 2, \dots, c\}$. For an event $A$, the indicator function $\sI (A)$ evaluates to 1 if $A$ takes place and 0 otherwise. The probability for an event $A (X)$ to occur is denoted by $\sP_{X\sim \gS} [A (X)]$, where $X$ is a random variable drawn from the distribution $\gS$.

Consider a model $g: \sR^d \to \sR^c$, whose components are $g_i: \sR^d \to \sR,\ i \in [c]$, where $d$ is the dimension of the input and $c$ is the number of classes. A classifier $f: \sR^d \to [c]$ can be obtained via $f(x) \in \argmax_{i\in[c]} g_i (x)$. In this paper, we assume that $g (\cdot)$ does not have the desired level of robustness, and refer to it as a ``standard classifier'' (as opposed to a ``robust classifier'' which we denote as $h (\cdot)$). Throughout this paper, we regard $g (\cdot)$ and $h (\cdot)$ as the base classifier logits. To denote their output probabilities, we use $\sigma \circ g (\cdot)$ and $\sigma \circ h (\cdot)$. Similarly, $\sigma \circ g_i (\cdot)$ denotes the predicted probability of the $i\th$ class from $g (\cdot)$. Moreover, we use $\gD$ to denote the set of all validation input-label pairs $(x_i, y_i)$.

We consider $\ell_p$-norm-bounded attacks on differentiable neural networks. A classifier $f (\cdot)$ is considered robust against adversarial perturbations at some input data $x \in \sR^d$ if it assigns the same label to all perturbed inputs $x+\delta$ such that $\norm{\delta}_p \leq \epsilon$, where $\epsilon \geq 0$ is the attack radius. We use PGD$_T$ to denote the $T$-step PGD attack.

\subsection{Related Adversarial Attacks and Defenses}

The fast gradient sign method (FGSM) and PGD attacks based on the first-order maximization of the cross-entropy loss have traditionally been considered classic and straightforward attacks \citep{Madry18, Goodfellow15}. However, these attacks have been shown to be insufficient as defenses designed against them are often easily circumvented \citep{Carlini17a, Athalye18a, Athalye18b, Papernot17}. To this end, various attack methods based on alternative loss functions, Expectation Over Transformation, and black-box perturbations have been proposed. Such efforts include MultiTargeted attack loss \citep{Gowal19}, AutoAttack \citep{Croce20a}, adaptive attack \citep{Tramer20}, minimal distortion attack \citep{Croce20b}, and many others, even considering attacking test-time defenses \citep{Croce22}. The diversity of attack methods has led to the creation of benchmarks such as RobustBench \citep{Croce20c} and ARES-Bench \citep{Liu23} to unify the evaluation of robust models.

On the defense side, while adversarial training \citep{Madry18} and TRADES \citep{Zhang19} have seen enormous success, such methods are often limited by a significantly larger amount of required training data \citep{Schmidt18}. Initiatives that construct more effective training data via data augmentation \citep{Rebuffi21, Gowal20, Gowal21} and generative models \citep{Sehwag22, Wang23} have successfully produced more accurate and robust models. Improved versions of adversarial training \citep{Jia22, Wang20, Shafahi19, Pagliardini22} have also been proposed.

Previous research has developed models that improve robustness by dynamically changing at test time. Specifically, Input-Adaptive Inference improves the accuracy-robustness trade-off by appending side branches to a single network, allowing for early-exit predictions \citep{Hu20}. Other initiatives that aim to enhance the accuracy-robustness trade-off include using the SCORE attack during training \citep{Pang22} and applying adversarial training for regularization \citep{Zheng21}.

Moreover, ensemble-based defenses, such as random ensemble \citep{Liu18}, diverse ensemble \citep{Pang19, Alam22, Adam20}, and Jacobian ensemble \citep{Co22}, have been proposed. In comparison, this work is distinct in that our mixing scheme uses two separate classifiers, incorporating one non-robust component while still ensuring the adversarial robustness of the overall design. By doing so, we take advantage of the high performance of modern pre-trained models, significantly alleviating the accuracy-robustness trade-off and achieving much higher overall performances. Additionally, unlike some previous ensemble initiatives, our formulation is deterministic and straightforward (in the sense of gradient propagation), making it easier to evaluate its robustness properly. The work \citep{Kumar22} also explored assembling an accurate classifier and a robust classifier, but the method considered robustness against distribution shift in a non-adversarial setting and was based on different intuitions. After the submission of this paper, the work \citep{Zhao23} also considered leveraging the power of a pair of standard and robust classifiers. However, instead of mixing the outputs, the authors proposed to distill a new model from the two base classifiers. While this approach also yielded impressive results, the distillation process is time-consuming.

\subsection{Locally Biased Smoothing} \label{sec:local_biased_smoothing}

Randomized smoothing, popularized by \citep{Cohen19}, achieves robustness at inference time by replacing the standard classifier $f (\cdot)$ with the smoothed model
\vspace{-6.6mm}
\begin{equation*}
	\widetilde{f} (x) \in \argmax_{i \in [c]} \sP_{\delta \sim \gS} \big[ f (x+\delta) = i \big],
	\vspace{-1mm}
\end{equation*}
where $\gS$ is a smoothing distribution, for which a common choice is a Gaussian distribution.

Note that $\gS$ is independent of the input $x$ and is often zero-mean.
The authors of \citep{Anderson21b} have shown that data-invariant smoothing enlarges the region of the input space at which the prediction of $\widetilde{f}(\cdot)$ stays constant.
Such an operation may unexpectedly degrade both clean and robust accuracy (the limiting case is when $\widetilde{f} (\cdot)$ becomes a constant classifier).
Furthermore, when $f (\cdot)$ is a linear classifier, the zero-mean restriction on $\gS$ leaves $f(\cdot)$ unchanged.
That is, randomized smoothing with a zero-mean distribution cannot help robustify even the most simple linear classifiers.
To overcome these limitations, \citep{Anderson21b} allowed $\gS$ to be input-dependent (denoted by $\gS_x$) and nonzero-mean and searched for distributions $\gS_x$ that best robustify $\tilde{f}(\cdot)$ with respect to the data distribution.
The resulting scheme is ``locally biased smoothing''.

It is shown in \citep{Anderson21b} that, up to a first-order linearization of the base classifier, the optimal locally biased smoothing distribution $\gS_x$ shifts the input point in the direction of its true class. Formally, for a binary classifier of the form $f(x) = \sign(g(x))$ with continuously differentiable $g (\cdot)$, maximizing the robustness of $\widetilde{f}(\cdot)$ around $x$ over all distributions $\gS_x$ with bounded mean yields the optimal locally biased smoothing classifier given by
\vspace{-1mm}
\begin{equation*}
	\widetilde{f}(x) = \sign(\widetilde{g}(x)), \quad \text{where} \quad \widetilde{g}(x) = g(x) + \gamma y(x) \norm{\nabla g(x)}_{p*},
\end{equation*}
where $y(x) \in \{-1,1\}$ is the true class of $x$, and where $\gamma \ge 0$ is the (fixed) bound on the distribution mean (i.e., $\norm{\mathbb{E}_{\delta \sim \gS_x} [ \delta ]}_p \le \gamma$).

Intuitively, this optimal locally biased smoothing classifier shifts the input along the direction $\nabla g(x)$ when $y(x) = 1$ as a means to make the classifier more likely to label $x$ into class $1$, and conversely shifts the input along the direction $-\nabla g(x)$ when $y(x)=-1$. Of course, during inference, the true class $y(x)$ is generally unavailable, and therefore \citep{Anderson21b} uses a ``direction oracle'' $h(x) \in \{-1,1\}$ as a surrogate for $y(x)$, resulting in the locally biased smoothing classifier
\begin{equation}
\begin{aligned}
	f^\gamma(x) &= \sign(h^\gamma(x)), \quad \text{where} \quad h^\gamma(x) = g(x) + \gamma h(x) \norm{\nabla g(x)}_{p*}.
\end{aligned}\label{eq:lbrs}
\end{equation}
Notice that unlike randomized smoothing, the computation \cref{eq:lbrs} is deterministic, which is a consequence of the closed-form optimization over $\gS_x$.

In contrast to data-invariant randomized smoothing, the direction oracle $h (\cdot)$ is learned from data, incorporating the data distribution into the manipulation of the smoothed classifier's decision boundaries.
This allows for increases in nonlinearity when the data implies that such nonlinearities are beneficial for robustness, resolving a fundamental limitation of data-invariant smoothing.
In general, the direction oracle should come from an inherently robust classifier.
Since such a robust model $h (\cdot)$ is often less accurate, the value $\gamma$ can be viewed as a trade-off parameter, as it encodes the amount of trust into the direction oracle.
The authors of \citep{Anderson21b} showed that when the direction oracle is a one-nearest-neighbor classifier, locally biased smoothing outperforms traditional randomized smoothing in binary classification.

\subsection{Adversarial Input Detectors}

Adversarial inputs can be detected via various methods.
For example, \citep{Metzen17} proposed to append an additional detection branch to an existing neural network and use adversarial data to train the detector in a supervised fashion.
However, \citep{Carlini17b} showed that it is possible to bypass this detection method. They constructed adversarial examples via the C\&W attacks \citep{Carlini17a} and simultaneously targeted the classification branch and the detection branch by treating the two branches as an ``augmented classifier''.
According to \citep{Carlini17b}, the detector is effective against the types of attack that it is trained with, but not necessarily the attack types that are absent in the training data.
It is thus reasonable to expect the detector to be able to detect a wide range of attacks if it is trained using sufficiently diverse types of attacks (including those targeting the detector itself).
While exhaustively covering the entire adversarial input space is intractable, and it is unclear to what degree one needs to diversify the attack types in practice, our experiments show that our modified architecture based on \citep{Metzen17} can recognize the SOTA AutoAttack adversaries with a high success rate.

The literature has also considered alternative detection methods that mitigate the above challenges faced by detectors trained in a supervised fashion \citep{Carrara19}.
Such initiatives include unsupervised detectors \citep{Aldahdooh21a, Aldahdooh21b} and re-attacking \citep{Ahmadi21}.
Since universally effective detectors have not yet been discovered, this paper focuses on transferring the properties of the existing detector toward better overall robustness.
Future advancements in the field of adversary detection can further enhance the performance of our method.

\section{Using a Robust Neural Network as the Smoothing Oracle} \label{sec:STD+ROB}

Locally biased smoothing was designed for binary classification, restricting its practicality. Here, we first extend it to the multi-class setting by treating the output $h^\gamma_i (x)$ of each class independently, giving rise to:
\begin{equation} \label{eq:adap_sm_1}
    \hismoa (x) \coloneqq g_i (x) + \gamma h_i (x) \norm{\nabla g_i (x)}_{p*}, \;\;\; \forall i \in [c].
\end{equation}

Note that if $\norm{\nabla g_i (x)}_{p*}$ is large for some $i$, then $\hismoa (x)$ can be large even if both $g_i (x)$ and $h_i (x)$ are small, potentially leading to incorrect predictions. To remove the effect of the magnitude difference across the classes, we propose a normalized formulation as follows:
\begin{equation} \label{eq:adap_sm_2}
    \hismob (x) \coloneqq \frac{g_i (x) + \gamma h_i (x) \norm{\nabla g_i (x)}_{p*}}{1 + \gamma \norm{\nabla g_i (x)}_{p*}}, \;\;\; \forall i \in [c].
\end{equation}

The parameter $\gamma$ adjusts between clean accuracy and robustness. It holds that $\hismob (x) \equiv g_i (x)$ when $\gamma = 0$, and $\hismob (x) \to h_i (x)$ when $\gamma \to \infty$ for all $x$ and all $i$.

With the mixing procedure generalized to the multi-class setting, we now discuss the choice of the smoothing oracle $h_i (\cdot)$. While $K$-NN classifiers are relatively robust and can be used as the oracle, their representation power is too weak. On the CIFAR-10 image classification task \cite{cifar10}, $K$-NN only achieves around $35 \%$ accuracy on clean test data. In contrast, an adversarially trained ResNet \citep{He16} can reach $50 \%$ accuracy on attacked test data \cite{Madry18}. This lackluster performance of $K$-NN becomes a significant bottleneck in the accuracy-robustness trade-off of the mixed classifier. To this end, we replace the $K$-NN model with a robust neural network. The robustness of this network can be achieved via various methods, including adversarial training, TRADES, and traditional randomized smoothing.

Further scrutinizing \cref{eq:adap_sm_2} leads to the question of whether $\norm{\nabla g_i (x)}_{p*}$ is the best choice for adjusting the mixture of $g (\cdot)$ and $h (\cdot)$.
This gradient magnitude term is a result of the setting of $h (x) \in \{-1, 1\}$ considered in \cite{Anderson21b}.
Here, we assume a different setting, where both $g (\cdot)$ and $h (\cdot)$ are multi-class and differentiable.
Thus, we further generalize the formulation to
\begin{gather} \label{eq:adap_sm_3}
    \hismoc (x) \coloneqq \frac{g_i(x) + \gamma R_i(x) h_i(x)}{1 + \gamma R_i(x)}, \;\;\; \forall i \in [c],
\end{gather}
where $R_i (x)$ is an extra scalar term that can potentially depend on both $\nabla g_i (x)$ and $\nabla h_i (x)$ to determine the ``trustworthiness'' of the base classifiers. Here, we empirically compare four options for $R_i (x)$, namely, $1$, $\norm{\nabla g_i (x)}_{p*}$, $\norm{\nabla \max_j g_j (x)}_{p*}$, and $\frac{\norm{\nabla g_i (x)}_{p*}} {\norm{\nabla h_i (x)}_{p*}}$. In \Cref{sec:R_options} in the supplemental materials, we explain how these four options were designed.

Another design choice is whether $g (\cdot)$ and $h (\cdot)$ should be the pre-softmax logits or the post-softmax probabilities. Note that since most attack methods are designed based on logits, the output of the mixed classifier should be logits rather than probabilities. This is because feeding output probabilities into attacks designed around logits effectively results in a redundant Softmax layer, which can cause gradient masking, an undesirable phenomenon that makes it hard to evaluate the proposed method's robustness properly. Thus, we have the following two options that make the mixed model compatible with existing gradient-based attacks:
\begin{enumerate}[leftmargin=8mm]
	\item Use the logits for both base classifiers, $g (\cdot)$ and $h (\cdot)$.
	\item Use the probabilities for both base classifiers, and then convert the mixed probabilities back to logits. The required ``inverse-softmax'' operator is simply the natural logarithm.
\end{enumerate}
\vspace{1mm}

\begin{figure}[t]
	\centering
	\vspace{-.5mm}
	\begin{minipage}{.54\textwidth}
		\includegraphics[width=\textwidth, height=.63\textwidth]{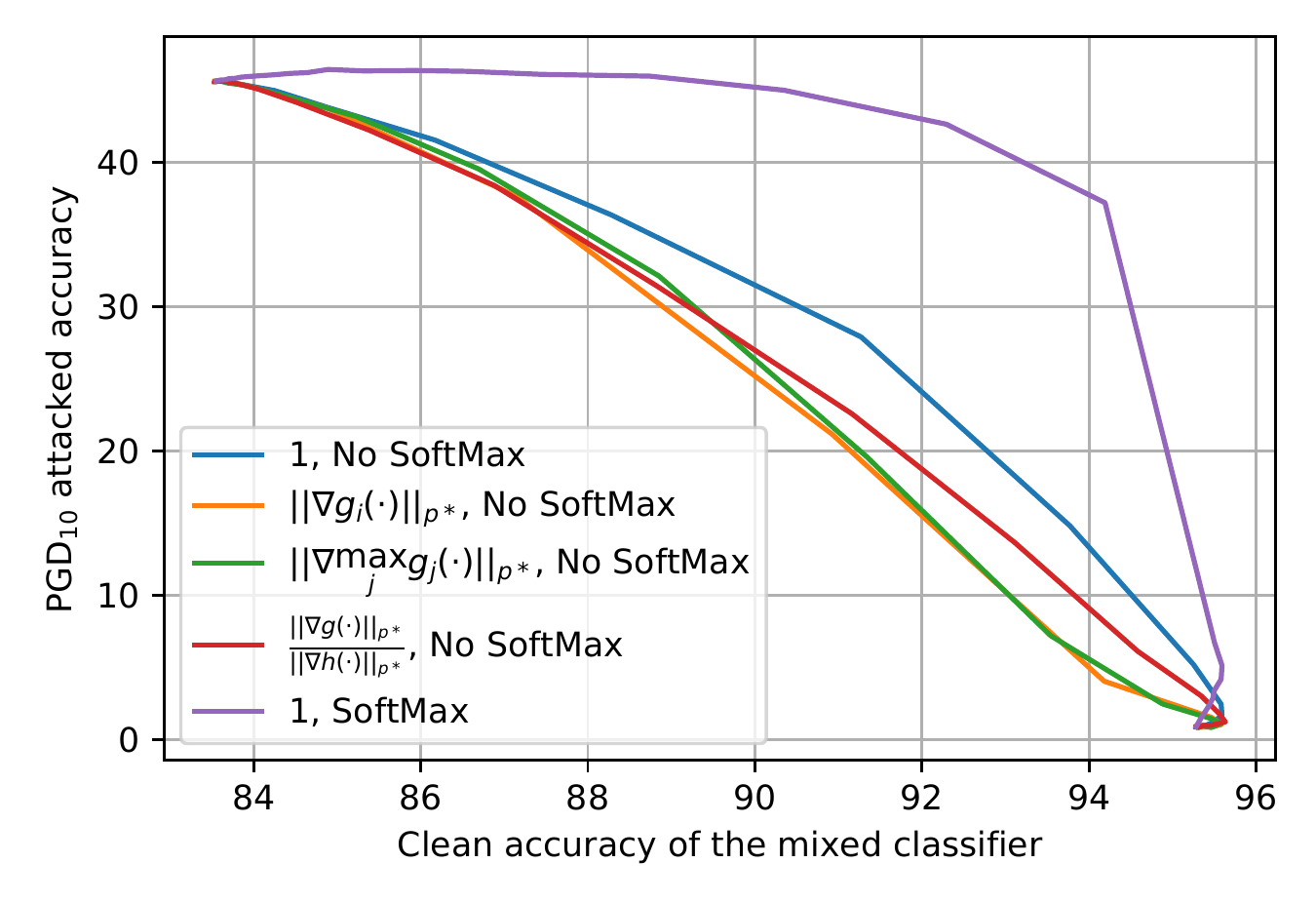}
	\end{minipage}
	\begin{minipage}{.41\textwidth}
		\begin{itemize}[leftmargin=4mm]
			\setlength\itemsep{.8em}
			\item \small ``No Softmax'' represents Option 1, i.e., use the logits $g (\cdot)$ and $h (\cdot)$.
			\item \small ``Softmax'' represents Option 2, i.e., use the probabilities $\sigma \circ g (\cdot)$ and $\sigma \circ h (\cdot)$.
			\item \small With the best formulation, high clean accuracy can be achieved with very little sacrifice on robustness.
		\end{itemize}
	\end{minipage}
	\vspace{-2mm}
	\caption{Compare the ``attacked accuracy -- clean accuracy'' curves for various $R_i (x)$ options.}
	\label{fig:compare_R}
	\vspace{-1mm}
\end{figure}

\Cref{fig:compare_R} visualizes the accuracy-robustness trade-off achieved by mixing logits or probabilities with different $R_i (x)$ options. Here, the base classifiers are a pair of standard and adversarially trained ResNet-18s. This ``clean accuracy versus PGD$_{10}$-attacked accuracy'' plot concludes that $R_i (x) = 1$ optimizes the accuracy-robustness trade-off, and $g (\cdot)$ and $h (\cdot)$ should be probabilities. \Cref{sec:more_compare_R} confirms this selection by repeating \Cref{fig:compare_R} with different model architectures,  other robust base model training methods, and various attack budgets.

Our selection of $R_i (x) = 1$ differs from $R_i (x) = \norm{g_i (x)}_{p*}$ used in \citep{Anderson21b}. Intuitively, \citep{Anderson21b} used linear classifier examples to motivate estimating the trustworthiness of the base models with their gradient magnitudes. However, when the base classifiers are highly nonlinear neural networks as in our case, while the local Lipschitzness of a base classifier still correlates with its robustness, its gradient magnitude is not always a good estimator of the local Lipschitzness. \Cref{sec:more_compare_R} provides additional discussions on this matter. Additionally, \Cref{sec:certified_radius_thms} offers theoretical intuitions for selecting mixing probabilities over mixing logits.

With these design choices implemented, the formulation \cref{eq:adap_sm_3} can be re-parameterized as
\begin{gather} \label{eq:adap_sm_4}
    \hialpha (x) \coloneqq \log \Big( (1 - \alpha) \sigma \circ g_i(x) + \alpha \cdot \sigma \circ h_i(x) \Big), \;\;\; \forall i \in [c],
\end{gather}
where $\alpha = \frac{\gamma} {1 + \gamma} \in [0, 1]$. We take $\halpha (\cdot)$ in \cref{eq:adap_sm_4}, which is a convex combination of base classifier probabilities, as our proposed mixed classifier. Note that \cref{eq:adap_sm_4} calculates the mixed classifier logits, acting as a drop-in replacement for existing models which usually produce logits. Removing the logarithm recovers the output probabilities without changing the predicted class.

\subsection{Theoretical Certified Robust Radius} \label{sec:certified_radius_thms}

In this section, we derive certified robust radii for the mixed classifier $\halpha (\cdot)$ introduced in \cref{eq:adap_sm_4}, given in terms of the robustness properties of $h (\cdot)$ and the mixing parameter $\alpha$.
The results ensure that despite being more sophisticated than a single model, $\halpha (\cdot)$ cannot be easily conquered, even if an adversary attempts to adapt its attack methods to its structure.
Such guarantees are of paramount importance for reliable deployment in safety-critical applications.
Note that while the focus of this paper is improved empirical accuracy-robustness trade-off and the existing literature often considers empirical and certified robustness separately, we will discuss how the certified results in this section provide important insights into the empirical performance, as the underlying assumptions are realistic and (approximately) verifiable for many empirically robust models.

Noticing that the base model probabilities satisfy $0 \leq \sigma \circ g_i (\cdot) \leq 1$ and $0 \leq \sigma \circ h_i (\cdot) \leq 1$ for all $i$, we introduce the following generalized and tightened notion of certified robustness.

\begin{definition} \label{def:robust_with_margin}
	Consider a model $h: \sR^d \to \sR^c$ and an arbitrary input $x \in \sR^d$. Further consider $y = \argmax_i h_i (x)$, $\mu \in [0, 1]$, and $r \ge 0$.
	Then, $h (\cdot)$ is said to be \textbf{certifiably robust at $x$ with margin $\mu$ and radius $r$} if $\sigma \circ h_y (x+\delta) \geq \sigma \circ h_i (x+\delta) + \mu$ for all $i \ne y$ and all $\delta \in \sR^d$ such that $\norm{\delta}_p \leq r$.
\end{definition}

Intuitively, \cref{def:robust_with_margin} ensures that all points within a radius from a nominal point have the same prediction as the nominal point, with the difference between the top and runner-up probabilities no smaller than a threshold.
For practical classifiers, the robust margin can be straightforwardly estimated by calculating the confidence gap between the predicted and the runner-up classes at an adversarial input obtained with strong attacks.
As shown in the experiments in \Cref{sec:conf_properties}, if a real-world robust model is robust at some input with a given radius, it is likely to be robust with a non-trivial margin.

\begin{lemma}
	\label{lem:certified_radius}
	Let $x \in \sR^d$ and $r \ge 0$.
	If it holds that $\alpha \in [\frac{1}{2}, 1]$ and $h (\cdot)$ is certifiably robust at $x$ with margin $\frac{1-\alpha} {\alpha}$ and radius $r$, then the mixed classifier $\halpha (\cdot)$ is robust in the sense that $\argmax_{i} \hialpha (x+\delta) = \argmax_{i} h_i (x)$ for all $\delta \in \sR^d$ such that $\norm{\delta}_p \leq r$.
\end{lemma}

\begin{proof}
Suppose that $h (\cdot)$ is certifiably robust at $x$ with margin $\frac{1-\alpha} {\alpha}$ and radius $r$.
Since $\alpha \in [\frac{1}{2}, 1]$, it holds that $\frac{1-\alpha} {\alpha} \in [0, 1]$. Let $y = \argmax_i h_i(x)$.
Consider an arbitrary $i \in [c] \setminus \{y\}$ and $\delta \in \sR^d$ such that $\norm{\delta}_p \le r$. Since $\sigma \circ g_i (x+\delta) \in [0, 1]$, it holds that
\begin{align*}
	\exp \left( \hyalpha (x+\delta) \right) - & \exp \left( \hialpha (x+\delta) \right) \\
	= & (1-\alpha) (\sigma \circ g_y (x+\delta) - \sigma \circ g_i (x+\delta)) + \alpha (\sigma \circ h_y (x+\delta) - \sigma \circ h_i (x+\delta)) \\
	\ge & (1-\alpha) (0-1) + \alpha (\sigma \circ h_y (x+\delta) - \sigma \circ h_i (x+\delta)) \\
	\ge & (\alpha - 1) + \alpha \left( \tfrac{1-\alpha} {\alpha} \right) = 0.
\end{align*}
Thus, it holds that $\hyalpha (x+\delta) \geq \hialpha (x+\delta)$ for all $i \neq y$, and thus $\argmax_i \hialpha (x+\delta) = y = \argmax_i h_i (x)$.
\end{proof}

While most existing provably robust results consider the special case with zero margin, we will show that models built via common methods are also robust with nonzero margins.
We specifically consider two types of popular robust classifiers: Lipschitz continuous models (\cref{thm:certified_radius}) and RS models (\cref{thm:randomized_smoothing}).
Here, \cref{lem:certified_radius} builds the foundation for proving these two theorems, which amounts to showing that Lipschitz and RS models are robust with nonzero margins and thus the mixed classifiers built with them are robust.
\cref{lem:certified_radius} can also motivate future researchers to develop margin-based robustness guarantees for base classifiers so that they immediately grant robustness guarantees for mixed architectures.

\cref{lem:certified_radius} additionally provides further justifications for using probabilities instead of logits in the smoothing operation.
Intuitively, $(1 - \alpha) \sigma \circ g_i (\cdot)$ is bounded between $0$ and $1 - \alpha$, so as long as $\alpha$ is relatively large (specifically, at least $\frac{1}{2}$), the detrimental effect of $g (\cdot)$ when subject to attack can be overcome by $h (\cdot)$.
Had we used the logits $g_i (\cdot)$, since this quantity cannot be bounded, it would have been much harder to overcome the vulnerability of $g (\cdot)$.

Since we do not make assumptions on the Lipschitzness or robustness of $g (\cdot)$, \cref{lem:certified_radius} is tight.
To understand this, we suppose that there exists some $i \in [c] \backslash \{ y \}$ and $\delta \neq 0$ such that $\norm{\delta}_p \leq r$ that make $\sigma \circ h_y (x+\delta) - \sigma \circ h_i (x+\delta) \coloneqq h_d$ smaller than $\frac{1-\alpha} {\alpha}$, indicating that $- \alpha h_d > \alpha-1$.
Since the only information about $g (\cdot)$ is that $\sigma \circ g_i (x+\delta) \in [0, 1]$ and thus the value $\sigma \circ g_y (x+\delta) - \sigma \circ g_i (x+\delta) \coloneqq g_d$ can be any number between $-1$ and $1$, it is possible that $(1 - \alpha) g_d$ is smaller than $-\alpha h_d$.
By \cref{eq:adap_sm_4}, when $(1 - \alpha) g_d < -\alpha h_d$, it holds that $\hyalpha (x+\delta) < \hialpha (x+\delta)$, and thus $\argmax_{i} \hialpha (x+\delta) \neq \argmax_{i} h_i (x)$.

\begin{definition}
	\label{def:lipschitz}
	A function $f \colon \sR^d \to \sR$ is called \emph{$\ell_p$-Lipschitz continuous} if there exists $L \in (0, \infty)$ such that $|f(x')-f(x)| \le L\|x'-x\|_p$ for all $x', x \in \sR^d$. The \textbf{Lipschitz constant} of such $f$ is defined to be
	\vspace{-1.9mm}
	\begin{equation*}
		\lip_p(f) \coloneqq \inf \left\{ L \in (0,\infty) : |f(x')-f(x)| \le L \|x'-x\|_p ~ \text{for all $x', x \in \sR^d$} \right\}.
	\end{equation*}
\end{definition}

\begin{assumption}
	\label{as:lipschitz}
	The base model $h (\cdot)$ is robust in the sense that, for all $i \in \{1, 2, \dots, n\}$, $\sigma \circ h_i (\cdot)$ is $\ell_p$-Lipschitz continuous with Lipschitz constant $\lip_p (\sigma \circ h_i)$.
\end{assumption}

\begin{theorem}
	\label{thm:certified_radius}
	Suppose that \cref{as:lipschitz} holds, and let $y = \argmax_i h_i(x)$, where $x \in \sR^d$ is arbitrary. Then, if $\alpha \in [\frac{1}{2}, 1]$, it holds that $\argmax_{i} \hialpha (x+\delta) = y$ for all $\delta \in \sR^d$ such that
	\begin{equation} \label{eq:lip_cert_rad}
		\bignorm{\delta}_p \le \rlippalpha (x) \coloneqq \min_{i \ne y} \frac{\alpha \big( \sigma \circ h_y(x) - \sigma \circ h_i(x) \big) + \alpha - 1} {\alpha \left( \lip_p (\sigma \circ h_y) + \lip_p (\sigma \circ h_i) \right)}.
	\end{equation}
\end{theorem}

\begin{proof}
	Suppose that $\alpha \in [\frac{1}{2}, 1]$, and let $\delta \in \sR^d$ be such that $\norm{\delta}_p \le \rlippalpha (x)$. Furthermore, let $i \in [c] \setminus \{y\}$. It holds that
	\begin{align*}
		\sigma \circ h_y (x+\delta) & - \sigma \circ h_i (x+\delta) \\
		& = \sigma \circ h_y(x) - \sigma \circ h_i(x) + \sigma \circ h_y(x+\delta) - \sigma \circ h_y(x) + \sigma \circ h_i(x) - \sigma \circ h_i(x+\delta) \\
		& \ge \sigma \circ h_y(x) - \sigma \circ h_i(x) - \lip_p (\sigma \circ h_y) \norm{\delta}_p - \lip_p (\sigma \circ h_i) \norm{\delta}_p \\
		& \ge \sigma \circ h_y(x) - \sigma \circ h_i(x) - \left( \lip_p (\sigma \circ h_y) + \lip_p (\sigma \circ h_i) \right) \rlippalpha (x) \ge \tfrac{1 - \alpha} {\alpha}.
	\end{align*}
	Therefore, $h (\cdot)$ is certifiably robust at $x$ with margin $\frac{1-\alpha} {\alpha}$ and radius $\rlippalpha (x)$. Hence, by \cref{lem:certified_radius}, the claim holds.
\end{proof}

Note that the $\ell_p$ norm that we certify can be arbitrary (e.g., $\ell_1$, $\ell_2$, or $\ell_\infty$), so long as the Lipschitz constant of the robust network $h (\cdot)$ is computed with respect to the same norm.

\cref{as:lipschitz} is not restrictive in practice.
For example, Gaussian RS with smoothing variance $\sigma^2 I_d$ ($I_d$ is the identity matrix in $\sR^{d \times d}$) yields robust models with $\ell_2$-Lipschitz constant $\sqrt{ \nicefrac{2} {\pi \sigma^2} }$ \cite{Salman19}.
In \Cref{sec:certified_exp}, we use experiments to verify the certified robustness of our method when $h (\cdot)$ is an RS model.
Additionally, methods have been proposed to compute upper bounds on neural network Lipschitz constants, thus allowing our certified robustness guarantees via \cref{as:lipschitz} and \cref{thm:certified_radius} to be employed \citep{Fazlyab19, Jordan20, Shi22}.
The notion of Lipschitz continuity has even motivated novel robustness methods \citep{Moosavi-Dezfooli19, Terjek20, Pfrommer24}.

\cref{as:lipschitz} can be relaxed to the even less restrictive scenario of using local Lipschitz constants over a neighborhood (e.g., a norm ball) around a nominal input $x$ (i.e., how flat $\sigma \circ h (\cdot)$ is near $x$) as a surrogate for the global Lipschitz constants.
In this case, \cref{thm:certified_radius} holds for all $\delta$ within this neighborhood.
Specifically, suppose that for an arbitrary input $x$ and an $\ell_p$ attack radius $\epsilon$, it holds that $\sigma \circ h_y (x) - \sigma \circ h_y (x + \delta) \le \epsilon \cdot \lip_p^x (\sigma \circ h_y)$ and $\sigma \circ h_i (x + \delta) - \sigma \circ h_i (x) \le \epsilon \cdot \lip_p^x (\sigma \circ h_i)$ for all $i \neq y$ and all perturbations $\delta$ such that $\norm{\delta}_p \leq \epsilon$.
Furthermore, suppose that the robust radius $\rlippalpha (x)$, as defined in \cref{eq:lip_cert_rad} but use the local Lipschitz constant $\lip_p^x$ as a surrogate to the global constant $\lip_p$, is not smaller than $\epsilon$.
Then, if the robust base classifier $h (\cdot)$ is correct at the nominal point $x$, then the mixed classifier $\halpha (\cdot)$ is robust at $x$ within the radius $\epsilon$.
The proof follows that of \cref{thm:certified_radius}.

The relaxed Lipschitzness defined above can be estimated for practical differentiable classifiers via an algorithm derived from the PGD attack \citep{Yang20}.
The authors of \citep{Yang20} showed that many existing empirically robust models, including those trained with AT or TRADES, are locally Lipschitz.
Note that \citep{Yang20} evaluates the local Lipschitz constants of the logits, whereas we analyze the probabilities, whose Lipschitz constants are much smaller, and small enough to certify a meaningful robust radius.
Hence, \cref{thm:certified_radius} provides important insights into the empirical robustness of the mixed classifier.
A detailed discussion is presented in \Cref{sec:est_lip}.

An intuitive explanation of \Cref{thm:certified_radius} is that if $\alpha$ approaches $1$, then $\rlippalpha (x)$ approaches $\min_{i \ne y} \frac{h_y(x) - h_i(x)}{\lip_p(h_y) + \lip_p(h_i)}$, which is the standard (global) Lipschitz-based robust radius of $h (\cdot)$ around $x$ (see, e.g., \cite{Fazlyab19, Hein17} for further discussions on Lipschitz-based robustness).
On the other hand, if $\alpha$ is too small compared to the relative confidence of $h (\cdot)$, namely, if there exists $i \ne y$ such that $\alpha \le \frac{1} {1 + \sigma \circ h_y(x) - \sigma \circ h_i(x)}$, then $\rlippalpha (x)$ is non-positive, and in this case we cannot provide non-trivial certified robustness for $\halpha (\cdot)$.
This is rooted in the fact that too small of an $\alpha$ value amounts to excess weight in the non-robust classifier $g (\cdot)$.
If $h (\cdot)$ is $100 \%$ confident in its prediction, then $\sigma \circ h_y(x) - \sigma \circ h_i(x) = 1$ for all $i \ne y$, and therefore this threshold value of $\alpha$ becomes $\frac{1}{2}$, leading to non-trivial certified radii for $\alpha > \frac{1}{2}$. 
However, once we put over $\frac{1}{2}$ of the weight into $g (\cdot)$, a nonzero radius around $x$ is no longer certifiable. Since there are no assumptions on the robustness of $g (\cdot)$ around $x$, this is intuitively the best one can expect, 

To summarize our certified robustness results, \cref{lem:certified_radius} shows the connection between the robust margin of the robust classifier and the robustness of the mixed classifier, while \cref{thm:certified_radius} demonstrates how general Lipschitz robust base classifiers exploit this relationship.
Since empirically robust models often satisfy the conditions of these two results, they guarantee that adaptive attacks cannot easily circumvent our proposed robustification.

In \Cref{sec:rs_radius} in the supplemental materials, we further tighten the certified radius estimation in the special case when $h (\cdot)$ is a randomized smoothing classifier and the robust radius is defined with the $\ell_2$ norm.
We achieve so by exploiting the stronger Lipschitzness of $x \mapsto \Phi^{-1} \big( \sigma \circ h_i(x) \big)$ arising from the unique structure granted by Gaussian convolution operations ($\Phi^{-1}$ is the inverse Gaussian cumulative distribution function).
In \Cref{sec:certified_exp}, we com- {\parfillskip=0pt \parskip=0pt \par}
\begin{wrapfigure}{r}[4mm]{.564\textwidth}
    \centering
    \includegraphics[width=.514\textwidth, trim={4mm 4.8mm 4mm 3.5mm}, clip] {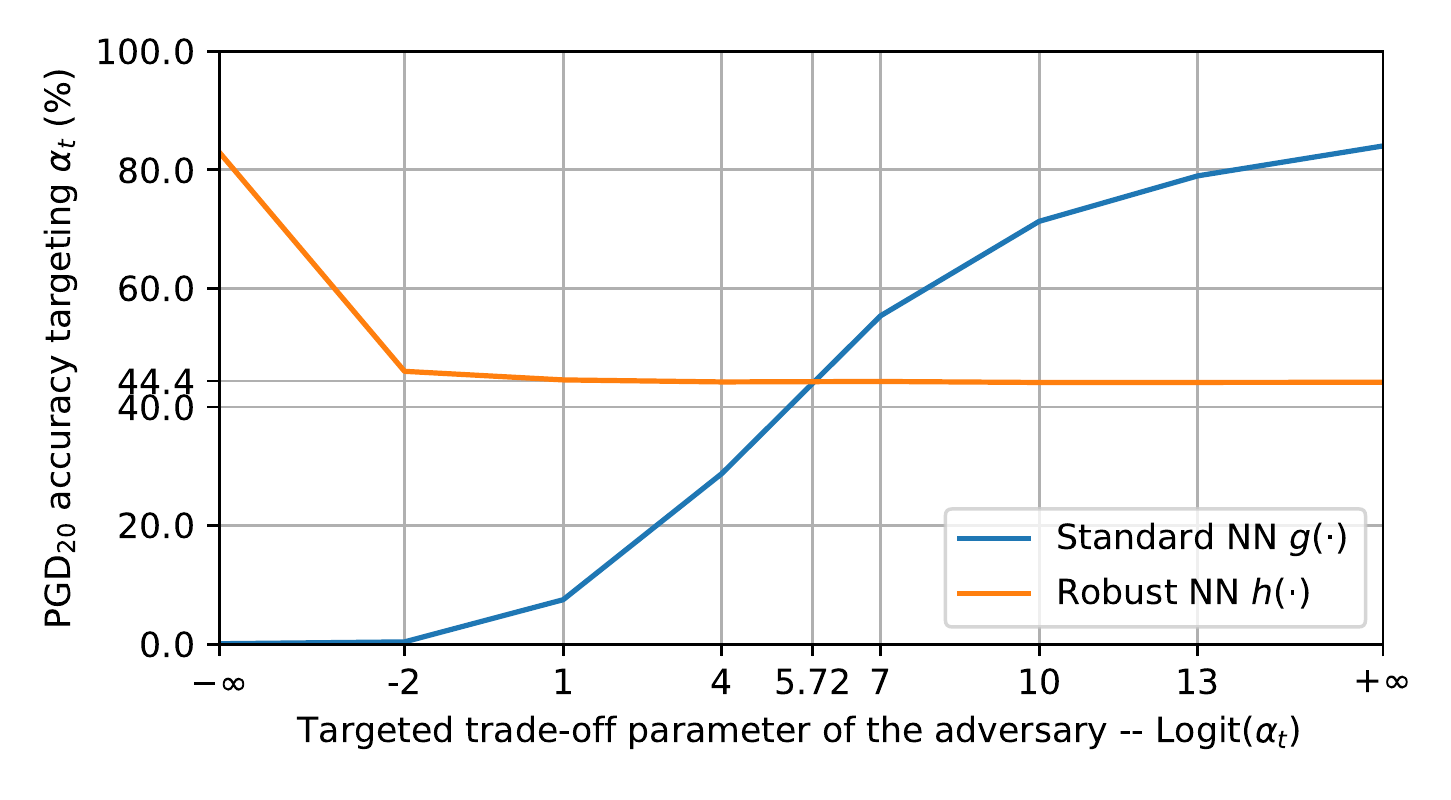}
    \captionsetup{width=.505\textwidth}
    \caption{Attacked accuracy of the accurate base classifier $g (\cdot)$ and the robust base model $h (\cdot)$ when the adversary targets different values of $\alpha_t$. For better readability, we use $\mathrm{Logit} (\alpha_t)$ as the horizontal axis labels, where $\mathrm{Logit} (\cdot)$ denotes the inverse function of Sigmoid.}
    \label{fig:ada_acc_2}
    \vspace{-1.5mm}
\end{wrapfigure}
\noindent pare the mixed classifier's certified robustness to existing certified methods.

\section{Adaptive Smoothing Strength with the Mixing Network} \label{sec:ada_smo}

So far, $\alpha$ has been treated as a fixed hyperparameter.
A more intelligent approach is to allow $\alpha$ to be different for each $x$ by using a function $\alpha (x)$.
We take $\alpha (x)$ to be deterministic, as stochastic defenses are challenging to properly evaluate.

One motivation for adopting the adaptive mixing ratio $\alpha (x)$ is that the optimal $\alpha^\star$ varies when $x$ changes.
For example, when $x$ is unperturbed, the standard model $g (\cdot)$ outperforms the robust base model $h (\cdot)$.
If $x$ is an attacked input targeting $g (\cdot)$, then $h (\cdot)$ should again be used. However, if the attack target is $h (\cdot)$, then as shown in \Cref{fig:ada_acc_2}, even though $h (\cdot)$ is robust, feeding $x$ into $g (\cdot)$ is a better choice.
This is because the vulnerabilities of $g (\cdot)$ and $h (\cdot)$ differ enough that an adversarial perturb targeting $h (\cdot)$ is benign to $g (\cdot)$.

When the adversary targets a mixed classifier $\halphat (\cdot)$, as $\alpha_t$ varies, the optimal strategy changes.
\Cref{fig:ada_acc_2} provides a visualization based on the CIFAR-10 dataset. Specifically, we assemble a composite model $\halphat (\cdot)$ using a ResNet-18 standard classifier $g (\cdot)$ and a ResNet-18 robust classifier $h (\cdot)$ (both from \citep{Na20}) via \cref{eq:adap_sm_4}.
Then, we attack $\halphat (\cdot)$ with different values of $\alpha_t$ via PGD$_{20}$, save the adversarial instances, and report the accuracy of $g (\cdot)$ and $h (\cdot)$ on these instances.
When $\alpha_t \leq \mathrm{Sigmoid} (5.72) = 0.9967$, the robust model $h (\cdot)$ performs better. When $\alpha_t > 0.9967$, the standard model $g (\cdot)$ is more suitable.

Throughout the remainder of this section, we overload the notation $\halpha (\cdot)$ even though $\alpha (\cdot)$ may be a function of the input, i.e., we define $\halpha (x) = h^{\alpha (x)}(x)$.

\subsection{The Existence of \texorpdfstring{$\alpha (x)$}{alpha (x)} that Achieves the Trade-Off} \label{sec:alpha(x)}

The following theorem shows that, under realistic conditions, there exists a function $\alpha (\cdot)$ that makes the combined classifier correct whenever either $g (\cdot)$ and $h (\cdot)$ makes the correct prediction, which further implies that the combined classifier matches the clean accuracy of $g (\cdot)$ and the attacked accuracy of $h (\cdot)$.

\begin{theorem} \label{THM:ADA}
    Let $\epsilon> 0$, $(x_1,y_1),(x_2,y_2)\sim \gD$, and $y_1\ne y_2$ (i.e., each input corresponds to a unique true label).
    Assume that $h_i(\cdot)$, $\|\nabla h_i(\cdot)\|_{p*}$, and $\|\nabla g_i(\cdot)\|_{p*}$ are all bounded and that there does not exist $z \in \sR^d$ such that $\norm{z - x_1}_p \leq \epsilon$ and $\norm{z - x_2}_p \leq \epsilon$.
    Then, there exists a function $\alpha (\cdot)$ such that the assembled classifier $\halpha (\cdot)$ satisfies \vspace{-.3mm}
    \begin{align*}
        \sP_{\substack{(x, y) \sim \gD \\ \delta \sim \gF}} \Big[ \argmax_{i \in [c]} \hialpha (x+\delta) = y \Big] 
        \geq \max \begin{Bmatrix}
            \sP_{(x, y) \sim \gD, \delta \sim \gF} \big[ \argmax_{i \in [c]} g_i (x+\delta) = y \big], \\
            \sP_{(x, y) \sim \gD, \delta \sim \gF} \big[ \argmax_{i \in [c]} h_i (x+\delta) = y \big]
        \end{Bmatrix},
    \end{align*}
    where $\gF$ is an arbitrary distribution that satisfies $\sP_{\delta \sim \gF} \big[ \norm{\delta}_p > \epsilon \big] = 0$.
\end{theorem}

\begin{proof}
Since it is assumed that the perturbation balls of the data are non-overlapping, the true label $y$ corresponding to each perturbed data $x+\delta$ with the property $\norm{\delta}_p \leq \epsilon$ is unique. Therefore, the indicator function \vspace{-.5mm}
\begin{align*}
    \alpha (x+\delta) = \bigg\{ \begin{array}{ll}
        0 & \text{if } \argmax_{i \in [c]} g_i (x+\delta) = y, \\
        1 & \text{otherwise},
    \end{array} \\[-8mm]
\end{align*}
satisfies that \vspace{-1.2mm}
\begin{align*}
    \alpha (x+\delta) = 1 \quad\; & \text{ if } \quad \argmax_{i \in [c]} g_i (x+\delta) \neq y \ \text{ and } \argmax_{i \in [c]} h_i (x+\delta) = y.
\end{align*}

\noindent Therefore, it holds that \vspace{-.6mm}
\begin{align*}
    \hialpha (x+\delta) = g_i (x+\delta) \quad\; & \text{ if } \quad \argmax_{i \in [c]} g_i (x+\delta) = y, \\[-1.2mm]
    \hialpha (x+\delta) = h_i (x+\delta) \quad\; & \text{ if } \quad \argmax_{i \in [c]} g_i (x+\delta) \neq y \ \text{ and } \argmax_{i \in [c]} h_i (x+\delta) = y, \\[-8.3mm]
\end{align*}
implying that \vspace{-1mm}
\begin{align*}
    & \argmax_{i \in [c]} \hialpha (x+\delta) = y \;\;\; \text{if} \;\;\; \Big( \argmax_{i \in [c]} g_i (x+\delta) = y \;\; \text{or} \; \argmax_{i \in [c]} h_i (x+\delta) = y \Big),
\end{align*}
which leads to the desired statement.	
\end{proof}

Note that the distribution $\gF$ is arbitrary, implying that the test data can be clean data, any type of adversarial data, or some combination of both.
As a special case, when $\gF$ is a Dirac measure at the origin, \Cref{THM:ADA} implies that the clean accuracy of $\halpha (\cdot)$ is as good as the standard classifier $g (\cdot)$.
Conversely, when $\gF$ is a Dirac measure at the worst-case perturbation, the adversarial accuracy of $\halpha (\cdot)$ is not worse than the robust model $h (\cdot)$, implying that if $h (\cdot)$ is inherently robust, then $\halpha (\cdot)$ inherits the robustness.
One can then conclude that there exists a $\halpha (\cdot)$ that matches the clean accuracy of $g (\cdot)$ and the robustness of $h (\cdot)$.

While \cref{THM:ADA} guarantees the existence of an instance of $\alpha (\cdot)$ that perfectly balances accuracy and robustness, finding an $\alpha (\cdot)$ that achieves this trade-off can be hard.
However, we will use experiments to show that an $\alpha (\cdot)$ represented by a neural network can retain most of the robustness of $h (\cdot)$ while greatly boosting the clean accuracy.
In particular, while we used the case of $\alpha (\cdot)$ being an indicator function to demonstrate the possibility of achieving the trade-off, \Cref{fig:compare_R} has shown that letting $\alpha$ take an appropriate value between $0$ and $1$ also improves the trade-off.
Thus, the task for the neural approximator is easier than representing the indicator function.
Also note that if certified robustness is desired, one can enforce a lower bound on $\alpha (\cdot)$ and take advantage of \cref{thm:certified_radius} while still enjoying the mitigated trade-off.

\subsection{Attacking the Adaptive Classifier} \label{sec:adaptive_attack}

When the combined model $\halpha (\cdot)$ is under adversarial attack, the function $\alpha (\cdot)$ provides an addition gradient flow path.
Intuitively, the attack should be able to force $\alpha$ to be small through this additional gradient path, tricking the mixing network into favoring the non-robust $g (\cdot)$.
Following the guidelines for constructing adaptive attacks \citep{Tramer20}, in the experiments, we consider the following types of attacks:
\begin{enumerate}[label=\textbf{\Alph*}, leftmargin=8mm]
\setlength\itemsep{1pt}
	\item \textbf{Gray-box PGD$_{20}$:} The adversary has access to the gradients of $g (\cdot)$ and $h (\cdot)$ when performing first-order optimization, but is not given the gradient of the mixing network $\alpha (\cdot)$.
	We consider untargeted PGD attack with a fixed initialization.
	\item \textbf{White-box PGD$_{20}$:} Since the mixed classifier is end-to-end differentiable, we follow \citep{Tramer20} and allow the attack to query end-to-end gradient, including that of the mixing network.
	\item \textbf{White-box AutoAttack:} AutoAttack is a stronger and more computationally expensive attack formed by an ensemble of four attack algorithms \citep{Croce20a}.
	It considers Auto-PGD (APGD) attacks with the untargeted cross-entropy loss and the targeted Difference of Logits Ratio loss, in addition to the targeted FAB attack and the black-box Square attack (SA) \citep{Andriushchenko20}.
	Again, the end-to-end mixed classifier gradient is available to the adversary.
	\item \textbf{Adaptive white-box AutoAttack:} Since the mixing network is a crucial component of the defense, we add an APGD loss component that aims to decrease $\alpha$ into AutoAttack to specifically target the mixing network.
\end{enumerate}

We will show that the adaptively smoothed model is robust against the attack that it is trained against. When trained using untargeted and targeted APGD$_{75}$ attacks, our model becomes robust against AutoAttack while noticeably improving the accuracy-robustness trade-off. In \Cref{sec:alpha_analysis}, we additionally consider evaluating with transfer attacks.

\subsection{The Mixing Network} \label{sec:mixing_network}

In practice, we use a neural network $\alpha_\theta (\cdot) : \sR^d \to [0, 1]$ to learn an effective mixing network that adjusts the outputs of $g (\cdot)$ and $h (\cdot)$.
Here, $\theta$ represents the trainable parameters of the mixing network, and we refer to $\alpha_\theta (\cdot)$ as the ``mixing network''.
To enforce an output range constraint, we apply a Sigmoid function to the mixing network output.
Note that when training the mixing network $\alpha_\theta (\cdot)$, the base classifiers $g (\cdot)$ and $h (\cdot)$ are frozen.
Freezing the base classifiers allows the mixed classifier to take advantage of existing accurate models and their robust counterparts, maintaining explainability and avoiding unnecessary feature distortions that the adversary can potentially exploit.

The mixing network's task of treating clean and attacked inputs differently is closely related to adversary detection. To this end, we adapt the detection architecture introduced in \citep{Metzen17} for our mixing network.
This architecture achieves high performance and low complexity, and is end-to-end differentiable, enabling convenient training and evaluation.
While \citep{Carlini17b} argued that simultaneously attacking the base classifier and the adversary detector can bring the detection rate of the detection method proposed in \citep{Metzen17} to near zero, we show that with several key modifications, the method is effective even against strong white-box attacks.
Specifically, our mixing network $\alpha_\theta (\cdot)$ takes advantage of both base models $g (\cdot)$ and $h (\cdot)$ by concatenating their intermediate features (\citep{Metzen17} only used one base model).
More importantly, we include stronger adaptive adversaries during training to generate much more diverse training examples.

\begin{figure*}
\vspace{-1mm}
    \centering
    \includegraphics[width=\textwidth]{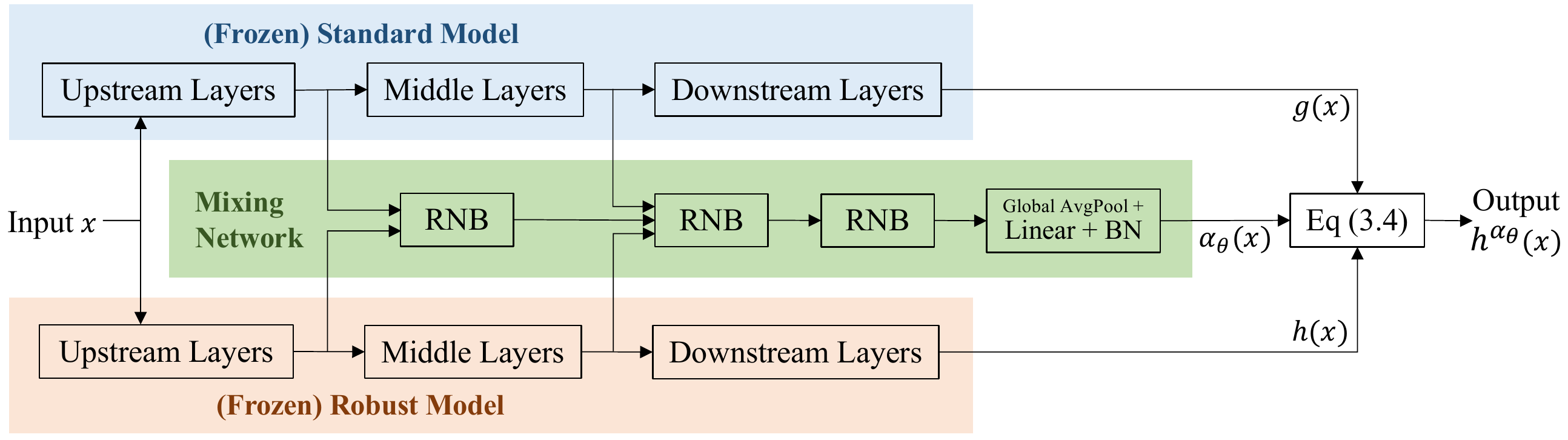}
    \caption{The overall architecture of the adaptively smoothed classifier introduced in \Cref{sec:ada_smo}. ``RNB'' stands for ResNetBlock and ``BN'' represents the 2D batch normalization layer.}
    \label{fig:mixing_arch}
    \vspace{-1mm}
\end{figure*}

The mixing network structure is based on a ResNet-18, which is known to perform well for a wide range of computer vision applications and is often considered the go-to architecture.
We make some minimal necessary changes to ResNet-18 for it to fit into our framework.
Specifically, as the mixing network takes information from both $g (\cdot)$ and $h (\cdot)$, it uses the concatenated embeddings from the base classifiers.
While \citep{Metzen17} considers a single ResNet as the base classifier and uses the embeddings after the first ResNet block, to avoid the potential vulnerability against ``feature adversaries'' \citep{Sabour15}, we consider the embeddings from two different layers of the base model.
\Cref{fig:mixing_arch} demonstrates the modified architecture.
The detailed implementations used in the experiment section are discussed in \Cref{sec:mixing_arch_rn}.

Since \Cref{fig:compare_R} shows that even a constant $\alpha$ can alleviate the accuracy-robustness trade-off, our method does not excessively rely on the performance of the mixing network $\alpha_\theta (\cdot)$.
In \Cref{sec:ada_exp}, we provide empirical results demonstrating that the above modifications help the overall mixed network defend against strong attacks.

\subsection{Training the Mixing Network} \label{sec:train_mixing_network}

Consider the following two loss functions for training the mixing network $\alpha_\theta (\cdot)$:
\begin{itemize}[leftmargin=8mm]
\setlength\itemsep{1pt}
    \item \textbf{Multi-class cross-entropy:} We minimize the multi-class cross-entropy loss of the combined classifier, which is the ultimate goal of the mixing network:
    \vspace{-.3mm}
    \begin{equation} \label{eq:unsup_opt}
        \min_\theta \E_{\substack{(x, y) \sim \gD \\ \delta \sim \gF}} \Big[ \lCE \big( \halphatheta (x+\delta), y \big) \Big],
    \end{equation}
    where $\lCE$ is the cross-entropy (CE) loss for logits and $y \in [c]$ is the label corresponding to $x$.
    The base classifiers $g (\cdot)$ and $h (\cdot)$ are frozen and not updated.
    Again, $\delta$ denotes the perturbation, and the distribution $\gF$ is arbitrary.
    In our experiments, to avoid overfitting to a particular attack radius, $\gF$ is formed by perturbations with randomized radii.  
    
    \item \textbf{Binary cross-entropy:} The optimal $\alpha^\star$ that minimizes $\lCE$ in \cref{eq:unsup_opt} can be estimated for each training point.
    Specifically, depending on whether the input is attacked and how it is attacked, either $g (\cdot)$ or $h (\cdot)$ should be prioritized.
    Thus, we treat the task as a binary classification problem and solve the optimization problem
    \vspace{-1.4mm}
    \begin{equation*}
        \min_\theta \E_{\substack{(x, y) \sim \gD \\ \delta \sim \gF}} \Big[ \lBCE \big( \alpha_\theta (x+\delta), \widetilde{\alpha} \big) \Big],
    \end{equation*}
    where $\lBCE$ is the binary cross-entropy (BCE) loss for probabilities and $\widetilde{\alpha} \in \{0, 1\}$ is the ``pseudo label'' for the output of the mixing network that approximates $\alpha^\star$.
\end{itemize}
\vspace{1mm}

Using only the multi-class loss suffers from a distribution mismatch between training and test data.
Specifically, the robust classifier $h (\cdot)$ may achieve a low loss on adversarial training data but a high loss on test data.
For example, with our ResNet-18 robust CIFAR-10 classifier, the PGD$_{10}$ adversarial training and test accuracy are very different, at $93.01 \%$ and $45.55 \%$ respectively.
As a result, approximating \cref{eq:unsup_opt} with empirical risk minimization on training data does not effectively optimize the true risk.
To understand this, notice that when the adversary perturbs a test input $x$ targeting $h (\cdot)$, the standard classifier prediction $g (x)$ yields a lower loss than $h (x)$.
However, if $x$ is an attacked example in the training set, then $g (x)$ and $h (x)$ have similar losses, and the mixing network does not receive an incentive to choose $g (\cdot)$ when detecting an attack targeting $h (\cdot)$.

The binary loss, on the other hand, does not capture the potentially different sensitivity of each input.
Certain inputs can be more vulnerable to adversarial attacks, and ensuring the correctness of the mixing network on these inputs is more crucial.

To this end, we combine the above two components into a composite loss function, incentivizing the mixing network to select the standard classifier $g (\cdot)$ when appropriate, while forcing it to remain conservative. The composite loss for each data-label pair $(x, y)$ is
\vspace{-.6mm}
\begin{align} \label{eq:comp_loss}
    \lcomp \big( \theta, ( x, y, \widetilde{\alpha} ) \big) = 
    \cCE \cdot \lCE \big( \halphatheta (x+\delta), y \big) + \cBCE \cdot \lBCE & \big( \alpha_\theta (x+\delta), \widetilde{\alpha} \big) \\
    \qquad + \cprod \cdot \lCE \big( \halphatheta (x+\delta), y \big) \cdot \lBCE & \big( \alpha_\theta (x+\delta), \widetilde{\alpha} \big), \nonumber
    \vspace{-2mm}
\end{align}
where the hyperparameters $\cCE, \cBCE$, and $\cprod$ control the weights of the loss components. \Cref{sec:loss_abla} in the supplemental materials discusses how these hyperparameters affect the performance of the trained mixing model.

\section{Numerical Experiments} \label{sec:experiments}

\subsection{Robust Neural Network Smoothing with a Fixed Strength} \label{sec:exp_CNN_fix}

We first consider the case where the smoothing strength $\alpha$ is a fixed value.
In this section, we focus on using empirically robust base classifiers and consider the CIFAR-10 dataset.
In \Cref{sec:certified_exp} in the supplemental materials, we present the certified robustness results when the robust base model is based on randomized smoothing, simultaneously instantiating the \Cref{lem:certified_radius} and \Cref{thm:certified_radius}.
In \Cref{sec:est_lip}, we show that empirically robust models can also take advantage of our theoretical analyses by estimating their Lipschitz constant.

\subsubsection{$\alpha$'s Influence on Mixed Classifier Robustness} \label{sec:alpha_analysis}

We first analyze how the accuracy of the mixed classifier changes with the mixing strength $\alpha$ under various settings.
Specifically, we consider PGD$_{20}$ attacks that target $g (\cdot)$ and $h (\cdot)$ individually (denoted as STD and ROB attacks), in addition to the adaptive PGD$_{20}$ attack generated using the end-to-end gradient of $\halpha (\cdot)$, denoted as the MIX attack.
Note that the STD and ROB attacks, which share the inspiration of \citep{Gao22}, correspond to the ``transfer attack'' setting, a common black-box attack strategy designed for defenses with unavailable or unreliable gradients.
Note that the models with the best transferability with the mixed classifier $\halpha (\cdot)$ would likely be its base classifiers $g (\cdot)$ and $h (\cdot)$, precisely corresponding to the STD and ROB attack settings.

We use a ResNet18 model trained on clean data as the standard base classifier $g (\cdot)$ and use another ResNet18 trained on PGD$_{20}$ data as the robust base classifier $h (\cdot)$.
The test accuracy corresponding to each $\alpha$ value is presented in \Cref{fig:STD+ROB}.
As $\alpha$ increases, the clean accuracy of $\halpha (\cdot)$ converges from the clean accuracy of $g (\cdot)$ to the clean accuracy of $h (\cdot)$.
In terms of the attacked performance, when the attack targets $g (\cdot)$, the attacked accuracy increases with $\alpha$.
When the attack targets $h (\cdot)$, the attacked accuracy decreases with $\alpha$, showing that the attack targeting $h (\cdot)$ becomes more benign when the mixed classifier emphasizes $g (\cdot)$.
When the attack targets $\halpha (\cdot)$, the attacked accuracy increases with $\alpha$.

When $\alpha$ is around $0.5$, the MIX-attacked accuracy of $\halpha (\cdot)$ quickly increases from near zero to more than $30 \%$ (which is two-thirds of $h (\cdot)$'s attacked accuracy).
This observation precisely matches the theoretical intuition provided by \Cref{thm:certified_radius}.
When $\alpha$ is greater than $0.5$, the clean accuracy gradually decreases at a much slower rate, leading to the noticeably alleviated accuracy-robustness trade-off.
Note that this improved trade-off is achieved without any further training beyond the weights of $g (\cdot)$ and $h (\cdot)$.
When $\alpha$ is greater than $0.55$, neither STD attack nor ROB attack can reduce the accuracy of the mixed classifier below the end-to-end gradient-based attack (MIX attack), indicating that the considered transfer attack is weaker than gradient-based attack for practical $\alpha$ values, and implying that the robustness of $\halpha (\cdot)$ does not rely on obfuscated gradients.
In \Cref{sec:conf_properties}, we will reveal that the source of $\halpha (\cdot)$'s robustness lies in $h (\cdot)$'s well-calibrated confidence properties.

\begin{figure}[t]
\centering
\begin{minipage}{.5\textwidth}
	\centering
	\vspace{-1.2mm}
    \includegraphics[width=\textwidth, height=.67\textwidth, trim={4.5mm 5mm 2.5mm 4mm}, clip] {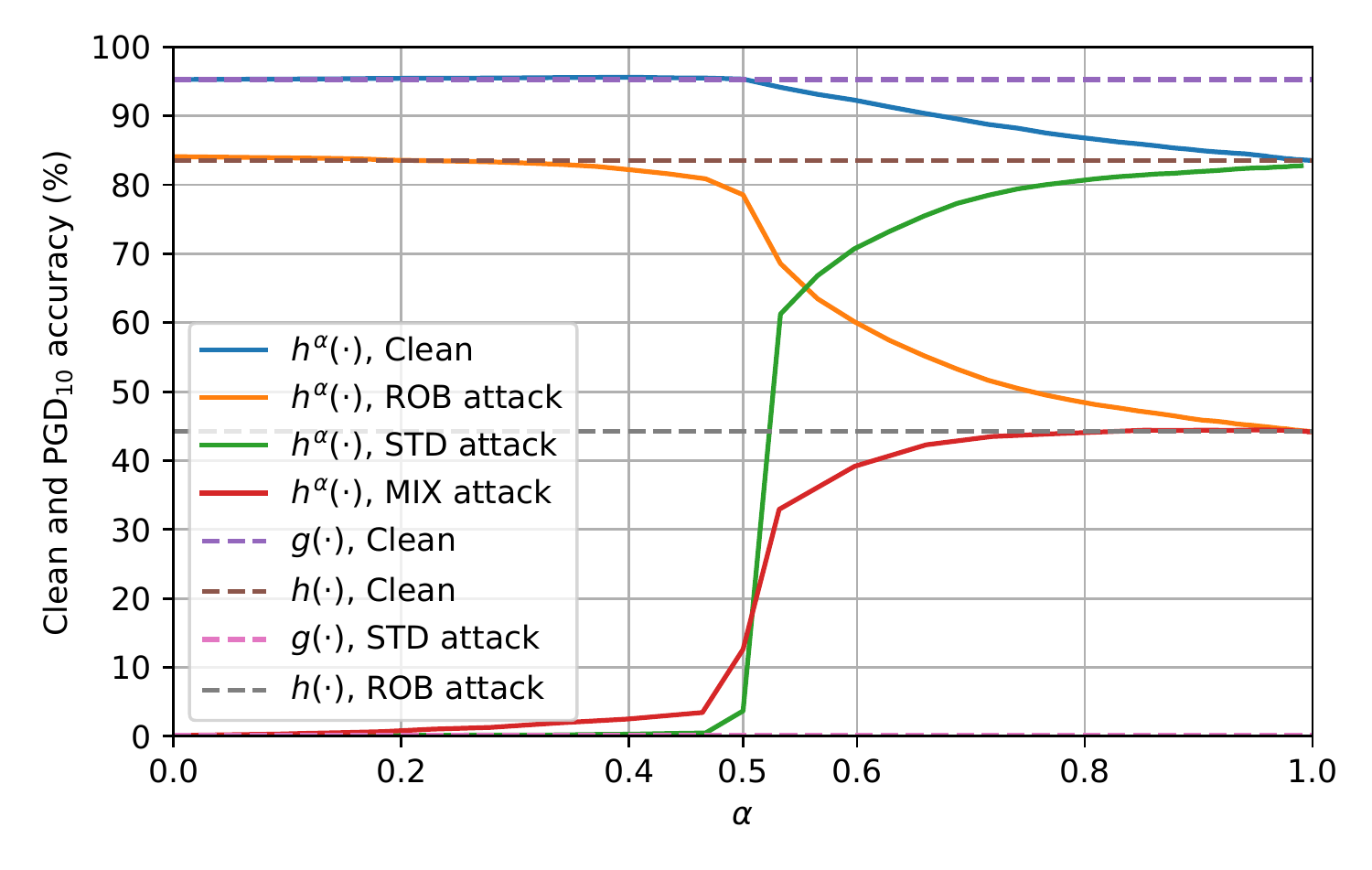}
    \caption{The performance of the mixed classifier $\halpha (\cdot)$. ``STD attack'', ``ROB attack'', and ``MIX attack'' refer to the PGD$_{20}$ attack generated using the gradient of $g (\cdot)$, $h (\cdot)$, and $\halpha (\cdot)$ respectively, with $\epsilon$ set to $\nicefrac{8}{255}$.}
    \label{fig:STD+ROB}
\end{minipage}
\hfill
\begin{minipage}{.47\textwidth}
	\centering
	\includegraphics[width=\textwidth, height=.62\textwidth, trim={4.5mm 5mm 2.5mm 4mm}]{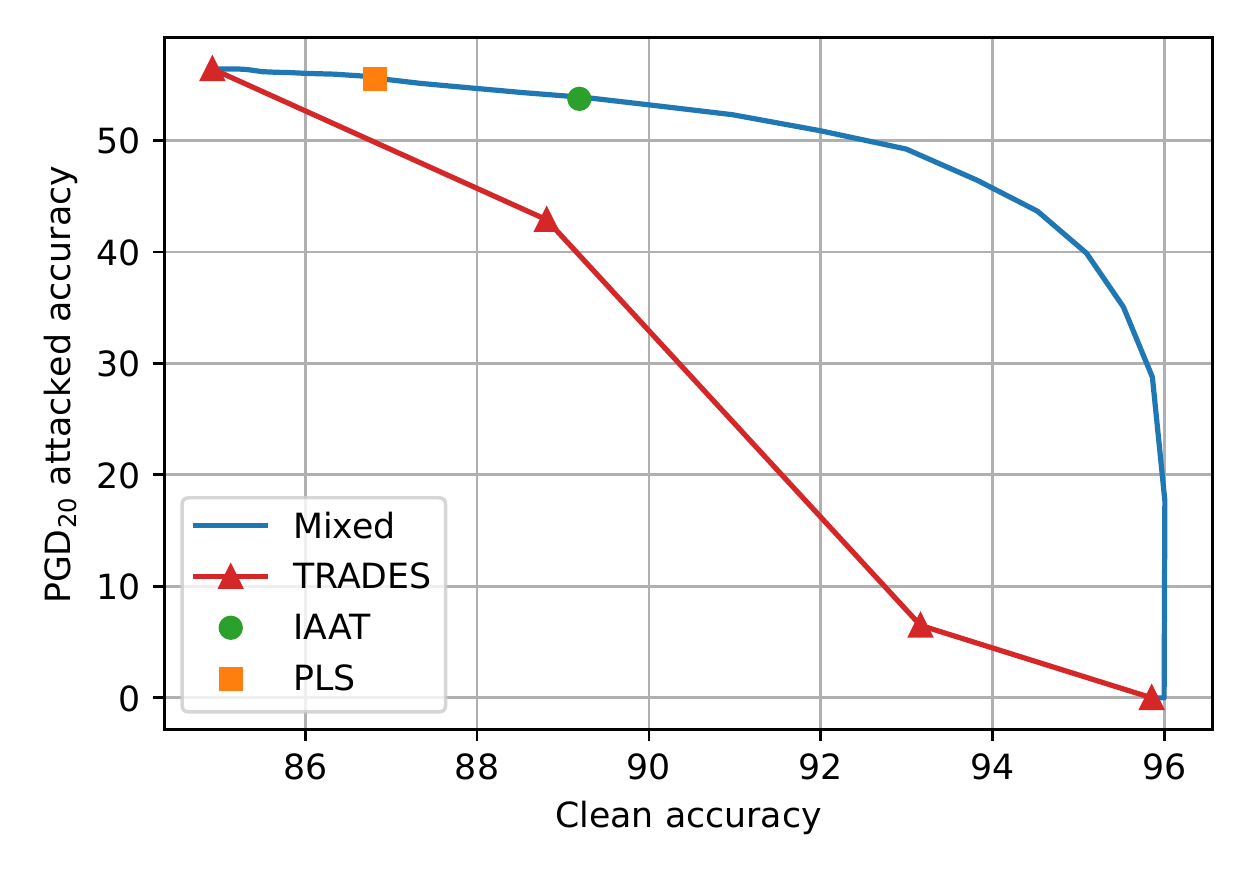}
	\caption{An accuracy-robustness trade-off comparison between our mixed classifier $\halpha (\cdot)$, denoted as ``Mixed'', and TRADES, IAAT, and PLS models.
	TRADES allows for sweeping between accuracy and robustness whereas IAAT and PLS are non-adjustable.}
	\label{fig:compare_w_trades}
\end{minipage}
\end{figure}

\subsubsection{The Relationship between $\halpha (\cdot)$'s Robustness and $h (\cdot)$'s Confidence} \label{sec:conf_properties}

Our theoretical analysis (\cref{lem:certified_radius}) has highlighted the relationship between the mixed classifier robustness and the robust base classifier $h (\cdot)$'s robust margin.
For practical models, the margin at a given radius can be estimated with the confidence gap between the predicted and runner-up classes evaluated on strongly adversarial inputs, such as images returned from PGD$_{20}$ or AutoAttack.
Moreover, the improved accuracy-robustness trade-off of the mixed classifier, as evidenced by the difference in how clean and attacked accuracy change with $\alpha$ in \Cref{fig:STD+ROB}, can also be explained by the prediction confidence of $h (\cdot)$.

According to \cref{tab:confidence}, the robust base classifier $h (\cdot)$ makes confident correct predictions even when under attack (average robust margin is $0.768$ evaluated with PGD$_{20}$ and $0.774$ with AutoAttack\footnote{The calculation details the AutoAttacked confidence gap are presented in \Cref{sec:autoattack_margin}.}).
Moreover, the robust margin of $h (\cdot)$ follows a long-tail distribution. Specifically, the median robust margin is $0.933$ (same number when evaluated with PGD$_{20}$ or AutoAttack), much larger than the $0.768/0.774$ average margin.
Thus, most attacked inputs correctly classified by $h (\cdot)$ are highly confident (i.e., robust with large margins), with only a tiny portion suffering from small robust margins.
As \cref{lem:certified_radius} suggests, such a property is precisely what adaptive smoothing relies on.
Intuitively, once $\alpha$ becomes greater than $0.5$ and gives $h (\cdot)$ more authority over $g (\cdot)$, $h (\cdot)$ can use its high confidence to correct $g (\cdot)$'s mistakes under attack.

On the other hand, $h (\cdot)$ is unconfident when it produces incorrect predictions on unattacked clean data, with the top two classes' output probabilities separated by merely $0.434$.
This probability gap again forms a long-tail distribution (the median is $0.378$ which is less than the mean), confirming that $h (\cdot)$ is generally unconfident when mispredicting and rarely makes confident incorrect predictions.
Now, consider clean data that $g (\cdot)$ correctly classifies and $h (\cdot)$ mispredicts.
Recall that we assume $g (\cdot)$ to be more accurate but less robust, so this scenario should be common.
Since $g (\cdot)$ is confident (average top two classes probability gap is $0.982$) and $h (\cdot)$ is usually unconfident, even when $\alpha > 0.5$ and $g (\cdot)$ has less authority than $h (\cdot)$ in the mixture, $g (\cdot)$ can still correct some of the mistakes from $h (\cdot)$.

In summary, $h (\cdot)$ is confident when making correct predictions on attacked data, enjoying the large robust margin required by \cref{lem:certified_radius}.
At the same time, $h (\cdot)$ is unconfident when misclassifying clean data, and such a confidence property is the key source of the mixed classifier's improved accuracy-robustness trade-off.
Additional analyses in \Cref{sec:more_compare_R} with alternative base models imply that multiple existing robust classifiers share the favorable confidence property and thus help the mixed classifier improve the trade-off.

The standard non-robust classifier $g (\cdot)$ often does not have this desirable property: even though it is confident on clean data as are robust classifiers, it also makes highly confident mistakes under attack.
Note that this does not undermine the mixed classifier robustness, since our formulation does not assume any robustness or smoothness from $g (\cdot)$.

\begin{table}[!tb]
	\vspace{-3mm}
	\aboverulesep=.2ex \belowrulesep=.4ex
	\caption{Average gap between the probabilities of the predicted class and the runner-up class. $g (\cdot)$ and $h (\cdot)$ are the same ones used in \Cref{fig:STD+ROB}. The confidence difference highlighted by the bold numbers is crucial to the mitigated accuracy-robustness trade-off of the mixed classifier.}
	\begin{small}
	\begin{center}
	\begin{tabular}{l!{\vrule width 2pt}c|c|c!{\vrule width 2pt}c|c|c!{\vrule width 2pt}c|c|c}
		\toprule
		& \multicolumn{3}{c!{\vrule width 2pt}}{Clean} & \multicolumn{3}{c!{\vrule width 2pt}}{PGD$_{20}$} & \multicolumn{3}{c}{AutoAttack} \\
		& Accuracy & \cmark\ Gap & \xmark\ Gap & Accuracy & \cmark\ Gap & \xmark\ Gap & Accuracy & \cmark\ Gap & \xmark\ Gap \\
		\midrule
		$g (\cdot)$ & $95.28 \%$ 		& $0.982$ 	& $0.698$			& $0.10 \%$
					& $0.602$			& $0.998$	& $0.00 \%$			& $-$
					& $0.986$ \\
		$h (\cdot)$ & $83.53 \%$ 		& $0.854$ 	& $\mathbf{0.434}$	& $44.17 \%$	
					& $\mathbf{0.768}$	& $0.635$	& $40.75 \%$			& $\mathbf{0.774}$
					& $0.553$ \\
		\bottomrule
	\end{tabular}
	\end{center}
	\vspace{1.5mm}
	\cmark\ Gap: The average gap between the confidences of the predicted class and the runner-up class among all correctly predicted validation data. \\
	\xmark\ Gap: The same quantity evaluated among all incorrectly predicted validation data.
	\end{small}
	\label{tab:confidence}
\end{table}

\subsubsection{Comparing the Accuracy-Robustness Trade-Off with Existing Methods} \label{sec:compare_with_trades}

This subsection compares the accuracy-robustness trade-off of the mixed classifiers with existing baseline methods that emphasize addressing this trade-off.

TRADES \citep{Zhang19} is one of the most famous and popular methods to improve the accuracy-robustness trade-off. Specifically, it trains robust models by minimizing the risk function \vspace{-.8mm}
\begin{equation*}
	\E_{(x, y) \sim \gD} \Big[ \lCE \big( h (x), y \big) + \beta \max_{\norm{\delta} \leq \epsilon} \lsurr \big( h (x+\delta), h (x) \big) \Big], \vspace{-.8mm}
\end{equation*}
where $\beta \geq 0$ is a trade-off parameter between the two loss components and $\lsurr$ is the ``surrogate loss'' that promotes robustness. The larger $\beta$ is, the more robust the trained model becomes at the expense of clean accuracy. By adjusting $\beta$, we can adjust the accuracy-robustness trade-off of TRADES similarly to adjusting $\alpha$ in our mixed classifier.

The authors of \citep{Zhang19} reported that $\beta = 6$ optimized the adversarial robustness and released the corresponding model.
We use this model and train three additional models with $\beta$ set to $0$, $0.1$, and $0.3$.
Here, $\beta = 0$ is standard training, and the other two numbers were chosen so that the model accuracy spreads relatively uniformly between $\beta = 0$ and $\beta = 6$.
All TRADES models use the WideResNet-34-10 architecture as in \citep{Zhang19}.
For a fair comparison, we build mixed classifiers using the TRADES model trained with $\beta = 0$ as $g (\cdot)$ and the $\beta = 6$ model as $h (\cdot)$.
We compare the relationship between the PGD$_{20}$ accuracy and the clean accuracy in \Cref{fig:compare_w_trades}.
Note that the trade-off curve of the mixed classifier intercepts the TRADES curve at the two ends (since the models are exactly the same at the two ends), and is significantly above the TRADES in the middle, indicating that the accuracy-robustness trade-off of the mixed classifier is much more benign than TRADES's.

IAAT \citep{Balaji19} and Properly Learned Smoothening (PLS) \citep{Chen20a} are two additional high-perfor-mance methods for alleviating the accuracy-robustness trade-off. Specifically, IAAT uses input-dependent attack budgets during adversarial training, while PLS performs stochastic weight averaging and smooths the logits via knowledge distillation and self-training. IAAT and PLS do not explicitly allow for adjusting between clean accuracy and adversarial robustness.

We implement IAAT on the same WideResNet-34-10 model architecture and add the result to \cref{fig:compare_w_trades}.
For PLS, we use the accuracy reported in \citep{Chen20a}.
It can be observed that the TRADES-based mixed classifier achieves a similar accuracy-robustness trade-off as IAAT and PLS, while allowing for sweeping between accuracy and robustness conveniently unlike previous models.
Note that for TRADES, adjusting the trade-off requires training a new model, which is costly.
Meanwhile, IAAT and PLS do not allow for explicitly adjusting the trade-off altogether (hence shown as single points in \Cref{fig:compare_w_trades}).
In contrast, for our mixing classifier, the trade-off can be adjusted at inference time by simply tuning $\alpha$ and does not require re-training.
Thus, our method is much more flexible and efficient while achieving a benign Pareto curve. 

Even though the clean-robust accuracy curve of adaptive smoothing overlaps with that of IAAT at a single point ($89.19 \%$ clean, $53.73 \%$ robust), adaptive smoothing still improves the overall accuracy-robustness trade-off.
Specifically, on top of IAAT's result, adaptive smoothing can further reduce the error rate by $31 \%$ while only sacrificing $6 \%$ of the robustness by achieving $\about 50 \% / \about 92.5 \%$ robust/clean accuracy.
In scenarios that are more sensitive to clean data performance, such a result makes adaptive smoothing more advantageous than IAAT, whose level of clean accuracy improvement is relatively limited.

Moreover, as discussed in \Cref{sec:intro} and confirmed in \Cref{sec:adap_exp_compare}, our mixed classifier can easily incorporate existing innovations that improve clean accuracy or adversarial robustness, whereas fusing these innovations into training-based methods such as TRADES, IAAT, and PLS can be much more complicated.
Also note that \Cref{fig:compare_w_trades} considers a constant $\alpha$ value, and adapting $\alpha$ for different input values further alleviates the trade-off.
To provide experimental evidence, in \Cref{fig:tradeoff_sota} in \Cref{sec:tradeoff_sota}, we add the mixed classifier results achieved with better base classifiers to the trade-off curve.

\subsection{Robust Neural Network Smoothing with Adaptive Strength} \label{sec:ada_exp}

Having validated the effectiveness of the mixing formulation described in \cref{eq:adap_sm_4}, we are now ready to incorporate the mixing network for adaptive smoothing strength.
As in \Cref{sec:ada_smo}, we denote the parameterized mixing network by $\alpha_\theta (\cdot)$, and slightly abuse notation by denoting the composite classifier with adaptive smoothing strength given by $\alpha_\theta (\cdot)$ by $\halphatheta (\cdot)$, which is defined by $\halphatheta (x) = h^{\alpha_\theta(x)} (x)$.

CIFAR-10 and CIFAR-100 are two of the most universal robustness evaluation datasets, and thus we use them to benchmark adaptive smoothing.
We consider $\ell_\infty$ attacks and use the AdamW optimizer \citep{Kingma15} for training the mixing network $\alpha_\theta (\cdot)$.
The training data for $\alpha_\theta (\cdot)$ include clean images and the corresponding types of attacked images (attack settings A, B, and C presented in \Cref{sec:adaptive_attack}).
For setting C (AutoAttack), the training data only includes targeted and untargeted APGD attacks, with the other two AutoAttack components, FAB and Square, excluded during training in the interest of efficiency but included for evaluation.
To alleviate overfitting, when generating training-time attacks, we randomize the attack radius and the number of steps, and add a randomly-weighted binary cross-entropy component that aims to decrease the mixing network output to the attack objective (thereby tricking it into favoring $g (\cdot)$).
Additionally, \Cref{sec:mixing_arch_rn} discusses the details of implementing the architecture in \Cref{fig:mixing_arch} for the ResNet base classifiers used in our experiments. \Cref{sec:loss_abla} conducts an ablation study on the hyperparameters in the composite loss function \cref{eq:comp_loss}.

\subsubsection{Ablation Studies Regarding Attack Settings} \label{sec:adap_exp_abla}

We first use smaller base classifiers to analyze the behavior of adaptive smoothing by exploring various training and attack settings.
The performance of the base models and the assembled mixed classifier are summarized in \cref{tab:cifar10}, where each column represents the performance of one mixed classifier.
The results show that the adaptive smoothing model can defend against the attacks on which the underlying mixing network is trained.
Specifically, for the attack setting A (gray-box PGD), $\halphatheta(\cdot)$ is able to achieve the same level of PGD$_{20}$-attacked accuracy as $h (\cdot)$ while retaining a similar level of clean accuracy as $g (\cdot)$.
For the setting B (white-box PGD), the attack is allowed to follow the gradient path provided by $\alpha_\theta (\cdot)$ and deliberately evade the part of the adversarial input space recognized by $\alpha_\theta (\cdot)$.
While the training task becomes more challenging, the improvement in the accuracy-robustness trade-off is still substantial.
Furthermore, the composite model can generalize to examples generated via the stronger AutoAttack.
For the setting C (AutoAttack), the difficulty of the training problem further escalates.
While the performance of $\halphatheta (\cdot)$ on clean data slightly decreases, the mixing network can offer a more vigorous defense against AutoAttack data, still improving the accuracy-robustness trade-off.

\cref{tab:cifar100} repeats the above analyses on the CIFAR-100 dataset.
The results confirm that adaptive smoothing achieves even more significant improvements on the CIFAR-100 dataset.
Notably, even for the most challenging attack setting C, $\halphatheta (\cdot)$ correctly classifies 1173 additional clean images compared with $h (\cdot)$ (cutting the error rate by a third) while making only 404 additional incorrect predictions on AutoAttacked inputs (increasing the error rate by merely 6.4 relative percent).
Such results show that $\alpha_\theta (\cdot)$ is capable of approximating a robust high-performance mixing network when trained with sufficiently diverse attacked data.
The fact that $\halphatheta (\cdot)$ combines the clean accuracy of $g (\cdot)$ and the robustness of $h (\cdot)$ highlights that our method significantly improves the accuracy-robustness trade-off.

\begin{table}[!tb]
	\centering
	\aboverulesep=.2ex \belowrulesep=.4ex
	\vspace{-3mm}
	\captionof{table}{CIFAR-10 results of adaptive smoothing models trained with three different settings.}
	\label{tab:cifar10}
	\vspace{-1.3mm}
	\begin{small}
		CIFAR-10 base classifier performances \\
		\begin{tabular}{l|l!{\vrule width 2pt}c|c|c}
			\toprule
			Model
			& \footnotesize{Architecture}
			& Clean			& PGD$_{20}$		& AutoAttack \\
			\midrule
			$g (\cdot)$ \scriptsize{(accurate)}
			& ResNet-18 \scriptsize{(Standard non-robust training)}
			& $95.28 \%$	& $0.12 \%$		& $0.00 \%$ \\
			$h (\cdot)$ \scriptsize{(robust)}
			& WideResNet-34-10 \scriptsize{(TRADES model \citep{Zhang19})}
			& $84.92 \%$ 	& $57.16 \%$		& $53.09 \%$ \\
			\bottomrule
		\end{tabular}
		
		\vspace{1.5mm}
		CIFAR-10 adaptive smoothing mixed classifier $\halphatheta (\cdot)$ performance \\
		\begin{tabular}{l!{\vrule width 2pt}c|c|c|c|c}
    		\toprule
    		Training Setting $\backslash$ Eval Data & Clean & \textbf{A} & \textbf{B} & \textbf{C} & \textbf{D} \scriptsize{(adaptive AutoAttack)} \\
    		\midrule
    		\textbf{A} \scriptsize{(gray-box PGD$_{20}$)}	& $92.05 \%$ & $57.22 \%$ & $56.63 \%$ & $40.04 \%$ & $39.85 \%$ \\
    		\textbf{B} \scriptsize{(white-box PGD$_{20}$)}	& $92.07 \%$ & $57.25 \%$ & $57.09 \%$ & $40.02 \%$ & $39.70 \%$ \\
    		\textbf{C} \scriptsize{(white-box AutoAttack)}	& $91.51 \%$ & $56.30 \%$ & $56.29 \%$ & $42.78 \%$ & $42.66 \%$ \\
    		\bottomrule
		\end{tabular}
	\end{small}
\end{table}

\begin{table}[!tb]
	\centering
	\aboverulesep=.2ex \belowrulesep=.4ex
	\vspace{-3.2mm}
	\captionof{table}{CIFAR-100 results of adaptive smoothing models trained with the three settings.}
	\label{tab:cifar100}
	\vspace{-1.3mm}
	\begin{small}
		CIFAR-100 base classifier performances \\
		\begin{tabular}{l|l!{\vrule width 2pt}c|c|c}
			\toprule
			Model & \footnotesize{Architecture} & Clean & PGD$_{20}$ & AutoAttack \\
			\midrule
			$g (\cdot)$ \scriptsize{(accurate)}	& ResNet-152 \scriptsize{(Based on BiT \citep{Kolesnikov20})}	& $91.38 \%$ & $0.14 \%$  & $0.00 \%$ \\
			$h (\cdot)$ \scriptsize{(robust)}	& WideResNet-70-16 \scriptsize{(From \citep{Gowal20})}		& $69.17 \%$ & $40.86 \%$ & $36.98 \%$ \\
			\bottomrule
		\end{tabular}
		
		\vspace{1.5mm}
		CIFAR-100 adaptive smoothing mixed classifier $\halphatheta (\cdot)$ performance \\
		\begin{tabular}{l!{\vrule width 2pt}c|c|c|c|c}
    		\toprule
    		Training Setting $\backslash$ Eval Data & Clean & \textbf{A} & \textbf{B} & \textbf{C} & \textbf{D} \scriptsize{(adaptive AutoAttack)} \\
    		\midrule
    		\textbf{A} \scriptsize{(gray-box PGD$_{20}$)}	& $83.99 \%$ & $40.04 \%$ & $30.59 \%$ & $23.54 \%$ & $23.78 \%$ \\
    		\textbf{B} \scriptsize{(white-box PGD$_{20}$)}	& $83.96 \%$ & $39.80 \%$ & $34.48 \%$ & $26.37 \%$ & $26.17 \%$ \\
    		\textbf{C} \scriptsize{(white-box AutoAttack)}	& $80.90 \%$ & $39.26 \%$ & $38.92 \%$ & $32.94 \%$ & $32.80 \%$ \\
    		\bottomrule
		\end{tabular}
	\end{small}
\end{table}

\subsubsection{Comparisons Against Existing SOTA Methods} \label{sec:adap_exp_compare}

\begin{table}[!tb]
\vspace{-3mm}
\aboverulesep=.2ex \belowrulesep=.4ex
\caption{Clean and AutoAttack (AA) accuracy of adaptive smoothing (AS) compared with the reported accuracy of previous models. AS clearly improves the accuracy-robustness trade-off.}
\label{tab:compare_cifar}
\vspace{-.5mm}

\begin{center}
	\normalsize{CIFAR-10} \\[1.2mm]
	\begin{minipage}{.54\textwidth}
	\centering
	\begin{footnotesize}
	\begin{tabular}{l!{\vrule width 2pt}c|c}
		\toprule
		Method											& Clean		 & AA \\
		\midrule
		AS (adaptive smoothing, ours) $^\star$			& $95.23 \%$ & $68.06 \%$ \\
		\midrule
		SODEF+TRADES \citep{Kang21}						& $93.73 \%$ & $71.28 \% \hspace{.4mm} ^\dagger \hspace{-1.8mm}$ \\
		Diffusion (EDM)+TRADES \citep{Wang23}			& $93.25 \%$ & $70.69 \%$ \\
		Diffusion (DDPM)+TRADES \citep{Rebuffi21}		& $92.23 \%$ & $66.58 \%$ \\
		TRADES XCiT-L12 \citep{Debenedetti22, Ali21}		& $91.73 \%$ & $57.58 \%$ \\
		Unlabeled data+TRADES \citep{Gowal20}			& $91.10 \%$ & $65.88 \%$ \\
		TRADES \citep{Gowal20}							& $85.29 \%$ & $57.20 \%$ \\
		\bottomrule
	\end{tabular}
	\vspace{1.5mm}
	\end{footnotesize}
	\end{minipage}
	\begin{minipage}{.45\textwidth}
		\centering
		\includegraphics[width=.98\textwidth]{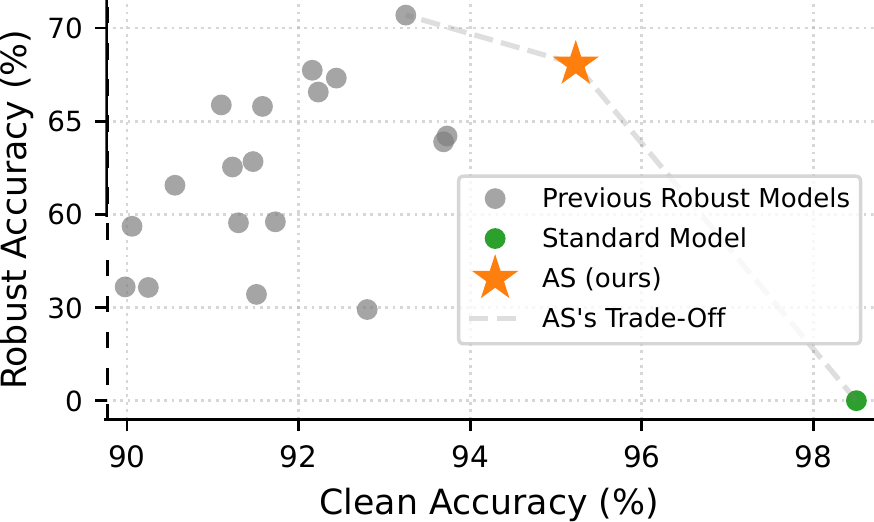}
		\vspace{-3mm}
	\end{minipage}
\end{center}
\vspace{3mm}

\small{$^\star$: Uses ``EDM + TRADES'' \citep{Wang23} as the robust base model $h (\cdot)$. $\hfill$} \\
\small{$^\dagger$: AutoAttack raises the ``potentially unreliable'' flag (explained in the next page), and adaptive attack} \\
\small{$\phantom{.}\hspace{2mm}$ reduces the attacked accuracy to $64.20 \%$. AutoAttack does not raise this flag for our models.}
\vspace{4mm}

\begin{center}
	\normalsize{CIFAR-100} \\[1.2mm]
	\begin{minipage}{.54\textwidth}
	\centering
	\begin{footnotesize}
	\begin{tabular}{l!{\vrule width 2pt}c|c}
		\toprule
		Method											& Clean		 & AA \\
		\midrule
		AS (adaptive smoothing, ours) $^\star$			& $85.21 \%$ & $38.72 \%$ \\
		AS (adaptive smoothing, ours) $^{\star \star}$	& $80.18 \%$ & $35.15 \%$ \\
		\midrule
		Diffusion (EDM)+TRADES \citep{Wang23}			& $75.22 \%$ & $42.67 \%$ \\
		Unlabeled data+TRADES \citep{Gowal20}			& $69.17 \%$ & $36.98 \%$ \\
		TRADES XCiT-L12 \citep{Debenedetti22, Ali21}		& $70.76 \%$ & $35.08 \%$ \\
		Diffusion (DDPM)+TRADES \citep{Rebuffi21}		& $63.56 \%$ & $34.64 \%$ \\
		SCORE Loss AT \citep{Pang22}						& $65.56 \%$ & $33.05 \%$ \\
		Diffusion (DDPM)+AT \citep{Sehwag22}				& $65.93 \%$ & $31.15 \%$ \\
		TRADES \citep{Gowal20}							& $60.86 \%$ & $30.03 \%$ \\
		\bottomrule
	\end{tabular}
	\end{footnotesize}
	\end{minipage}
	\begin{minipage}{.45\textwidth}
		\centering
		\includegraphics[width=.98\textwidth]{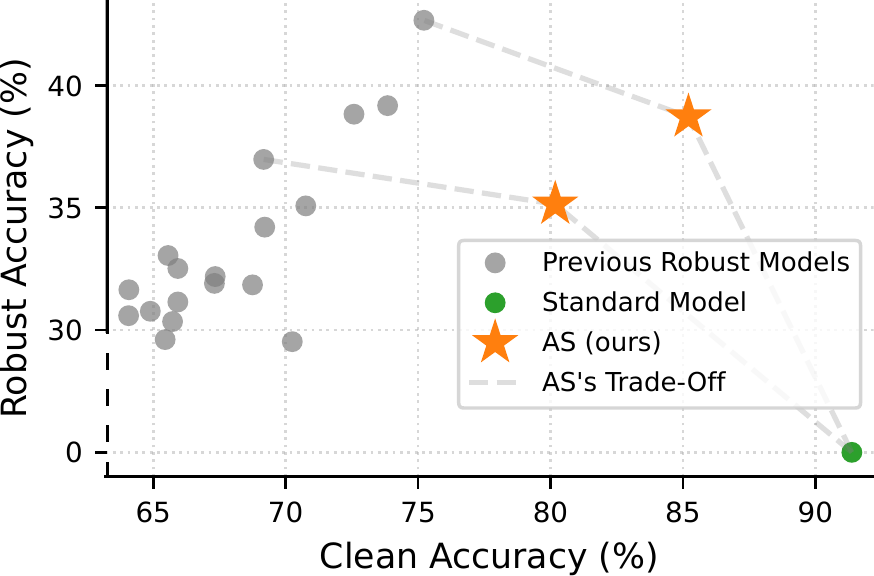}
		\vspace{-3mm}
	\end{minipage}
	\vspace{1mm}
\end{center}

\small{$^{\phantom{\star} \star}$: Uses ``EDM+TRADES'' \citep{Wang23} as the robust base model $h (\cdot)$. $\hspace{17.6mm}$} \\
\small{$^{\star \star}$: Uses ``Unlabeled data+TRADES'' \citep{Gowal20} as the robust base model $h (\cdot)$. $\hspace{4mm}$}
\end{table}

In this section, we use \cref{tab:compare_cifar} to show that when using SOTA base classifiers, adaptive smoothing noticeably improves the accuracy-robustness trade-off over existing methods.

Since the literature has regarded AutoAttack \citep{Croce20a} as one of the most reliable robustness evaluation methods (weaker attacks such as PGD are known to be circumventable), we select AutoAttack-evaluated robust models as baselines.
We highlight that these baseline models should not be treated as competitors, since advancements in building robust classifiers can be incorporated into our framework as $h (\cdot)$, helping adaptive smoothing perform even better.

For the accurate base classifier $g (\cdot)$, we fine-tune the BiT ResNet-152 checkpoint (from \citep{Kolesnikov20}, pre-trained on ImageNet-21k) on CIFAR-10 or CIFAR-100.
Following the recipe from \citep{Kolesnikov20}, our CIFAR-10 model achieves a $98.50 \%$ clean accuracy and our CIFAR-100 model achieves $91.38 \%$.

For CIFAR-10, we select the robust model checkpoint released in \citep{Wang23} as the robust base classifier $h (\cdot)$.
Compared with $h (\cdot)$, adaptive smoothing retains 96.3 (relative) percent of the robust accuracy while reducing the clean data error rate by 29.3 (relative) percent.
Among all models available on RobustBench as of submission, our method achieves the third highest AutoAttacked accuracy, only behind \citep{Wang23} (used as $h (\cdot)$ in our model) and \citep{Kang21} (for which AutoAttack is unreliable and the best-known attacked accuracy is lower than ours).
Meanwhile, the clean accuracy of our mixed classifier is higher than all listed models with non-trivial $\ell_\infty$ robustness and is even higher than the listed non-robust model that uses standard training.

While the above results demonstrate reconciled accuracy and robustness, the clean accuracy improvement over existing works may not seem highly prominent.
Note that our method is still highly effective in this setting, but its efficacy is not fully reflected in the numbers.
This is because SOTA robust base classifiers are already highly accurate on the easier CIFAR-10 dataset, almost matching standard models' clean accuracy \citep{Rebuffi21, Gowal20, Gowal21}, leaving not much room for improvements.
However, the accuracy-robustness trade-off remains highly penalizing for more challenging tasks such as CIFAR-100, for which existing robust models suffer significant accuracy degradation.
As existing methods for improving standard model accuracy may not readily extend to robust ones, training-based trade-off alleviation struggles on harder tasks, making it particularly advantageous to mix already-trained classifiers via adaptive smoothing.
We now support this claim with more significant improvements on CIFAR-100.

For CIFAR-100, we consider two robust base models and build two adaptive smoothing mixed classifiers.
Compared with their corresponding robust base models, both mixed classifiers improve the clean accuracy by ten percentage points while only losing four points in AutoAttacked accuracy.
As of the submission of this paper, the mixed classifier whose robust base model is from \citep{Wang23} achieved an AutoAttacked accuracy better than any other methods on RobustBench \citep{Croce20c}, except \citep{Wang23} itself.
Simultaneously, this mixed model offers a clean accuracy improvement of ten percentage points over any other listed models.
These results confirm that adaptive smoothing significantly alleviates the accuracy-robustness trade-off.

We also report that the SA component of AutoAttack, which performs gradient-free black-box attacks on images that gradient-based attack methods fail to perturb, only changes very few predictions.
Specifically, AutoAttack will raise a ``potentially unreliable'' flag if SA further reduces the accuracy by at least $0.2$ percentage points.
This flag is not thrown for our models in \cref{tab:compare_cifar}, indicating that the mixed classifiers' robustness is not a result of gradient obfuscation.
Thus, gradient-based attacks in AutoAttack sufficiently evaluate our models.

\section{Conclusions}

This paper proposes ``adaptive smoothing'', a flexible framework that leverages the mixture of the output probabilities from an accurate model and a robust model to mitigate the accuracy-robustness trade-off of neural classifiers.
We use theoretical and empirical observations to motivate our design, and mathematically prove that the resulting mixed classifier can inherit the robustness of the robust base model under realistic assumptions.
We then adapt an adversarial input detector into a (deterministic) mixing network, further improving the accuracy-robustness trade-off.
Solid empirical results confirm that our method can simultaneously benefit from the high accuracy of modern pre-trained standard (non-robust) models and the robustness achieved via SOTA robust classification methods. 

Because our theoretical studies demonstrate the feasibility of leveraging the mixing network to eliminate the accuracy-robustness trade-off, future advancements in adversary detection can further reconcile this trade-off via our framework.
Moreover, the proposed method conveniently extends to various robust base models and attack types/budgets.
Thus, this work paves the way for future research to focus on accuracy or robustness without sacrificing the other.

\vspace{2cm}
{\small
\bibliographystyle{siamplain}
\bibliography{journal}
}

\newpage
\appendix

\section{Additional Theoretical and Experimental Results on Certified Robustness} \label{sec:certified_results}
\vspace{-2mm}

In this section, we tighten the certified radius in the special case when $h (\cdot)$ is a randomized smoothing classifier and the robust radii are defined in terms of the $\ell_2$ norm.
This enables us to visualize and compare the certified robustness of the mixed classifier to existing certifiably robust methods in \Cref{sec:certified_exp}.

\subsection{Larger Certified Robust Radius for Randomized Smoothing Base Classifiers} \label{sec:rs_radius}

Since randomized smoothing often operates on the probabilities and does not consider the logits, with a slight abuse of notation, we use $h (\cdot)$ to denote the probabilities throughout this section (as opposed to denoting the logits in the main text).

\begin{assumption}
	\label{as:randomized_smoothing}
	The classifier $h (\cdot)$ is a (Gaussian) randomized smoothing classifier, i.e., $h (x) = \E_{\xi \sim \gN (0, \sigma^2 I_d)} \left[ \overline{h} (x+\xi) \right]$ for all $x \in \sR^d$, where $\overline{h} \colon \sR^d \to [0, 1]^c$ is the output probabilities of a neural model that is non-robust in general. Furthermore, for all $i \in [c]$, $\overline{h}_i (\cdot)$ is not 0 almost everywhere or 1 almost everywhere.
\end{assumption}

\begin{theorem}
	\label{thm:randomized_smoothing}
	Suppose that \Cref{as:randomized_smoothing} holds, and let $x \in \sR^d$ be arbitrary. Let $y = \argmax_i h_i(x)$ and $y' = \argmax_{i \ne y} h_i(x)$. Then, if $\alpha \in [\frac{1}{2}, 1]$, it holds that $\argmax_i \hialpha (x+\delta) = y$ for all $\delta \in \sR^d$ such that
	\begin{align*}
		\norm{\delta}_2 & \le \rsigmaalpha (x) \coloneqq \frac{\sigma}{2} \Big( \Phi^{-1} \big( \alpha h_y(x) \big) - \Phi^{-1} \big( \alpha h_{y'} (x) + 1 - \alpha \big) \Big).
	\end{align*}
\end{theorem}

\begin{proof}
	First, note that since every $\overline{h}_i (\cdot)$ is not 0 almost everywhere or 1 almost everywhere, it holds that $h_i (x) \in (0, 1)$ for all $i$ and all $x$. Now, suppose that $\alpha \in [\frac{1}{2}, 1]$, and let $\delta\in \sR^d$ be such that $\norm{\delta}_2 \le \rsigmaalpha (x)$. Let $\mu_\alpha \coloneqq \frac{1-\alpha}{\alpha}$ (conversely, $\alpha = \frac{1} {\mu_\alpha+1}$). We construct a scaled classifier $\tilde{h} \colon \sR^d \to \sR^c$, whose $i\th$ entry is defined as
	\begin{equation*}
		\tilde{h}_i (x) = \begin{aligned}
			\begin{cases}
				\frac{\overline{h}_y (x)} {\onemualpha} \hspace{5.3mm} = \alpha \overline{h}_y(x) & \text{if $i = y$}, \\
				\frac{\overline{h}_i (x) + \mu_\alpha} {\onemualpha} = \alpha \overline{h}_i (x) + 1 - \alpha & \text{if $i \ne y$}.
			\end{cases}
			\end{aligned}
	\end{equation*}
	Furthermore, define a scaled RS classifier $\hat{h} \colon \sR^d \to \sR^c$ based on $\tilde{h}_i (\cdot)$ by 
	\begin{equation*}
		\hat{h} (x) = \E_{\xi \sim \gN (0, \sigma^2 I_d)} \left[ \tilde{h} (x+\xi) \right].
	\end{equation*}
	Then, since it holds that
	\begin{align*}
		\tilde{h}_y(x) &= \frac{\overline{h}_y (x)} {\onemualpha} \in \left( 0, \frac{1}{\onemualpha} \right) \subseteq (0, 1), \\
		\tilde{h}_i(x) &= \frac{\overline{h}_i (x) + \mu_\alpha} {\onemualpha} \in \left( \frac{\mu_\alpha} {\onemualpha}, 1 \right) \subseteq (0, 1), \;\; \forall i \neq y,
	\end{align*}
	it must be the case that $0 < \tilde{h}_i (x) < 1$ for all $i$ and all $x$, and hence, for all $i$, the function $x \mapsto \Phi^{-1} \big( \hat{h}_i(x) \big)$ is $\ell_2$-Lipschitz continuous with Lipschitz constant $\frac{1}{\sigma}$ (see \cite[Lemma~1]{Levine19}, or Lemma 2 in \cite{Salman19} and the discussion thereafter). Therefore,
	\begin{equation}
		\left| \Phi^{-1} \big( \hat{h}_i (x+\delta) \big) - \Phi^{-1} \big( \hat{h}_i (x) \big) \right| \le \frac{\norm{\delta}_2} {\sigma} \le \frac{\rsigmaalpha (x)}{\sigma} \label{eq:lipschitz_inequality}
	\end{equation}
	for all $i$. Applying \cref{eq:lipschitz_inequality} for $i=y$ yields that
	\begin{equation} \label{eq:lip_y}
		\Phi^{-1} \big( \hat{h}_y (x+\delta) \big) \ge \Phi^{-1} \big( \hat{h}_y (x) \big) - \frac{\rsigmaalpha (x)} {\sigma},
	\end{equation}
	and, since $\Phi^{-1}$ is monotonically increasing and $\hat{h}_i(x) \le \hat{h}_{y'} (x)$ for all $i \ne y$, applying \cref{eq:lipschitz_inequality} to $i \ne y$ gives that
	\begin{align} \label{eq:lip_not_y}
		\Phi^{-1} \big( \hat{h}_i (x+\delta) \big) & \le \Phi^{-1} \big( \hat{h}_i (x) \big) + \frac{\rsigmaalpha (x)} {\sigma} \le \Phi^{-1} \big( \hat{h}_{y'} (x) \big) + \frac{\rsigmaalpha (x)} {\sigma}.
	\end{align}
	Subtracting \cref{eq:lip_not_y} from \cref{eq:lip_y} gives that
	\begin{align*}
		\Phi^{-1} \big( \hat{h}_y (x+\delta) \big) & - \Phi^{-1} \big( \hat{h}_i (x+\delta) \big) \ge \Phi^{-1} \big( \hat{h}_y (x) \big) - \Phi^{-1} \big( \hat{h}_{y'} (x) \big) - \frac{2 \rsigmaalpha (x)} {\sigma}
	\end{align*}
	for all $i \ne y$. By the definitions of $\mu_\alpha$, $\rsigmaalpha (x)$, and $\hat{h} (x)$, the right-hand side of this inequality equals zero, and hence, since $\Phi$ is monotonically increasing, we find that $\hat{h}_y (x+\delta) \ge \hat{h}_i (x+\delta)$ for all $i \ne y$. Therefore,
	\begin{align*}
		\frac{h_y(x+\delta)} {\onemualpha} = & \E_{\xi \sim \gN (0, \sigma^2 I_d)} \left[ \frac{\overline{h}_y (x+\delta+\xi)} {\onemualpha} \right] = \hat{h}_y (x+\delta) \\
		\ge & \hat{h}_i (x+\delta) = \E_{\xi \sim \gN (0, \sigma^2 I_d)} \left[ \frac{\overline{h}_i (x+\delta+\xi) + \mu_\alpha} {\onemualpha} \right] = \frac{h_i (x+\delta) + \mu_\alpha} {\onemualpha}.
	\end{align*}
	Hence, $h_y (x+\delta) \ge h_i (x+\delta) + \mu_\alpha$ for all $i \ne y$, so $h (\cdot)$ is certifiably robust at $x$ with margin $\mu_\alpha = \frac{1-\alpha} {\alpha}$ and radius $\rsigmaalpha (x)$. Therefore, by \cref{lem:certified_radius}, it holds that $\argmax_i \hialpha (x+\delta) = y$ for all $\delta \in \sR^d$ such that $\norm{\delta}_2 \le \rsigmaalpha (x)$, which concludes the proof.
\end{proof}

\subsection{Experiment Setup} \label{sec:rs_exp_setup}

Before visualizing the certified robustness results, we first explain the experiment setup.
We let the smoothing strength $\alpha$ be a fixed value.
Since a (Gaussian) RS model with smoothing covariance matrix $\sigma^2 I_d$ has an $\ell_2$-Lipschitz constant $\sqrt{\nicefrac{2}{\pi\sigma^2}}$, such a model can be used to simultaneously visualize both theorems, with \cref{thm:randomized_smoothing} giving tighter certificates of robustness.
Consider the CIFAR-10 dataset. We select $g (\cdot)$ to be an ImageNet-pretrained ResNet-152 model with a clean accuracy of $98.50 \%$ (the same one used in \cref{tab:compare_cifar}), and use the RS models presented in \citep{Zhang19} as $h (\cdot)$.

Notice that it is possible to maintain the mixed classifier's clean accuracy while changing its robustness behavior by jointly adjusting the mixing weight $\alpha$ and the RS variance $\sigma^2$.
Specifically, increasing $\sigma^2$ certifies larger radii at the cost of decreased clean accuracy.
To compensate, $\alpha$ can be reduced to allow more emphasis on the accurate base classifier $g (\cdot)$, thereby restoring the clean accuracy.
We want to understand how jointly adjusting $\alpha$ and $\sigma^2$ affects the certified robustness property while fixing the clean accuracy.
To this end, for a fair comparison, for the mixed classifier $\halpha (\cdot)$, we select an $\alpha$ value such that the clean accuracy of $\halpha (\cdot)$ matches that of another RS model $\hbase (\cdot)$ with a smaller smoothing variance.

The expectation term in the RS formulation is approximated with the empirical mean of 10,000 random perturbations\footnote{The authors of \citep{Cohen19} show that 10,000 Monte Carlo samples are sufficient to provide representative results.} drawn from $\gN(0, \sigma^2 I_d)$. 
The certified radii of $\hbase(\cdot)$ are calculated using Theorems \ref{thm:certified_radius} and $\ref{thm:randomized_smoothing}$ by setting $\alpha$ to $1$.

Note that our certified robustness results make no assumptions on the accurate base classifier $g (\cdot)$, and do not depend on it in any way.
Hence, to achieve the best accuracy-robustness trade-off, we should select a model with the highest clean accuracy as $g (\cdot)$.
Using a more accurate $g (\cdot)$ will allow using a larger $\alpha$ value for the same level of clean accuracy, thereby indirectly improving the certified robustness of the mixed classifier.
Such a property allows the mixed classifier to take advantage of state-of-the-art standard (non-robust) classifiers.
In contrast, since these models are often not trained for the purpose of RS, directly incorporating them into RS may produce suboptimal results.
Therefore, our mixed classifier has better flexibility and compatibility, even in the certified robustness setting.

Additionally, since we make no assumptions on the confidence properties of $g (\cdot)$, we replace the Softmax operation in \cref{eq:adap_sm_4} with a ``Hardmax''.
I.e., the confidence of $g (\cdot)$ used in the mixture is a one-hot vector associated with $g (\cdot)$'s predicted class.
Note that this replacement is equivalent to applying a temperature scaling of zero to $g (\cdot)$.
By doing so, the mixed classifier's clean accuracy can be enhanced (because the higher-accuracy base model is made more confident) while not affecting the certified robustness (because they do not depend on $g (\cdot)$.

\subsection{Visualization of the Certified Robust Radii} \label{sec:certified_exp}

\begin{figure*}
	\centering
	\includegraphics[width=.98\textwidth, trim={4mm 5mm 3mm 2mm}, clip]{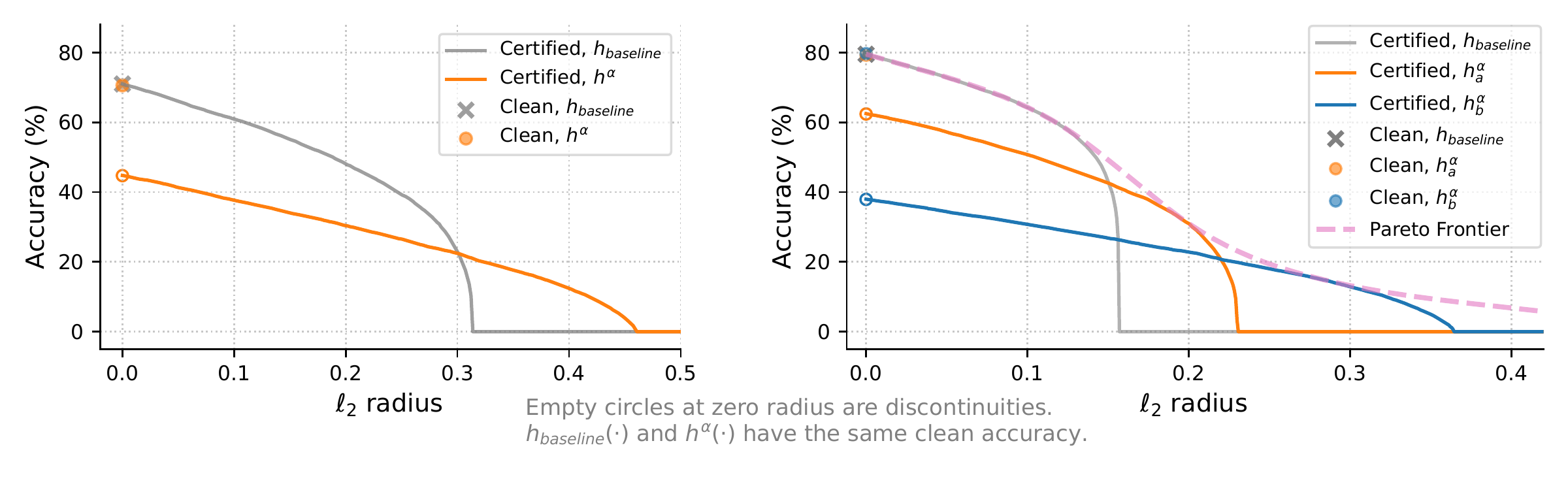}
	\begin{subfigure}[t]{.42\textwidth}
		\centering
		\captionsetup{justification=centering}
		\vspace{-5.5mm}
		\caption{\footnotesize $\hbase (\cdot)$: RS with $\sigma=0.5$. $\hspace{-2.9mm}$ \\
		$ $ \\[.3mm]
		$\halpha (\cdot)$ uses $\alpha=0.79$; $\hspace{4.44mm}$ \\
		$h (\cdot)$ is RS with $\sigma=1$. $\hspace{1.3mm}$}
		\vspace{-2mm}
	\end{subfigure}
	\hfill
	\begin{subfigure}[t]{.56\textwidth}
		\centering
		\captionsetup{justification=centering}
		\vspace{-5.5mm}
		\caption{\footnotesize $\hbase (\cdot)$: RS with $\sigma=0.25$. $\hspace{26.4mm}$ \\[.3mm]
		Consider two mixed classifier examples: $\hspace{7mm}$ \\
		$\halpha_a (\cdot)$ uses $\alpha=0.79$ and $h_a (\cdot)$ is RS with $\sigma=0.5$; $\hspace{-6.7mm}$ \\
		$\hspace{.2mm} \halpha_b (\cdot)$ uses $\alpha=0.71$ and $h_b (\cdot)$ is RS with $\sigma=1.0$. $\hspace{-6.6mm}$}
		\vspace{-2mm}
	\end{subfigure}
	\caption{Closed-form certified accuracy of RS models and our mixed classifier with the Lipschitz-based bound in \Cref{thm:certified_radius}. The mixed classifier can optimize the certified robust accuracy at each radius without affecting clean accuracy by tuning $\alpha$ and $\sigma^2$. The resulting Pareto frontier demonstrates significantly extended certified radii over a standalone RS model, signaling improved accuracy-robustness trade-off.}
    \label{fig:certified_pareto}
    \vspace{-2.5mm}
\end{figure*}

We are now ready to visualize the certified robust radii presented in \cref{thm:certified_radius} and \cref{thm:randomized_smoothing}.
\Cref{fig:certified_pareto} displays the calculated certified accuracies of $\halpha (\cdot)$ and $\hbase (\cdot)$ at various attack radii.
The ordinate ``Accuracy'' at a given abscissa ``$\ell_2$ radius'' reflects the percentage of the test data for which the considered model gives a correct prediction and a certified radius at least as large as the $\ell_2$ radius under consideration.

In both subplots of \Cref{fig:certified_radii}, the clean accuracy is the same for $\hbase (\cdot)$ and $\halpha (\cdot)$.
Note that the certified robustness curves of $\halpha (\cdot)$ do not connect to the clean accuracy when $\alpha$ approaches zero.
This discontinuity occurs because Theorems \ref{thm:certified_radius} and \ref{thm:randomized_smoothing} both consider robustness with respect to $h (\cdot)$ and do not issue certificates to test inputs at which $h (\cdot)$ makes incorrect predictions, even though $\halpha (\cdot)$ may correctly predict at some of these points in reality.
This is reasonable because we do not assume any robustness or Lipschitzness of $g (\cdot)$, and $g (\cdot)$ is allowed to be arbitrarily incorrect whenever the radius is non-zero.

\begin{figure*}[t]
	\centering
	\includegraphics[width=.98\textwidth, trim={4mm 5mm 3mm 2mm}, clip]{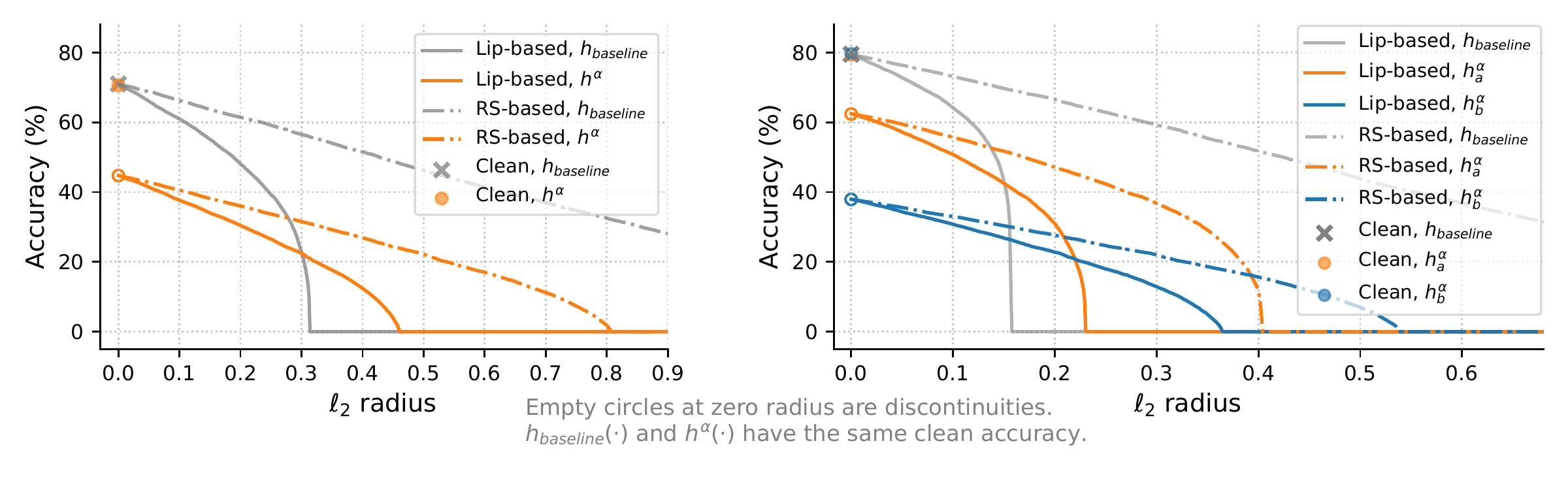}
	\caption{Tightening the certified robustness bounds with RS-based (\Cref{thm:randomized_smoothing}) certificates. The models are the same ones as in \Cref{fig:certified_pareto}.}
	\label{fig:certified_radii}
	\vspace{-1.5mm}
\end{figure*}

The Lipschitz-based bound of \cref{thm:certified_radius} allows us to visualize the performance of the mixed classifier $\halpha (\cdot)$ when $h (\cdot)$ is an $\ell_2$-Lipschitz model.
In this case, the curves associated with $\halpha (\cdot)$ and $\hbase (\cdot)$ intersect, with $\halpha (\cdot)$ achieving higher certified accuracy at larger radii and $\hbase (\cdot)$ certifying more points at smaller radii.
By jointly adjusting $\alpha$ and the Lipschitz constant of $h (\cdot)$, it is possible to change the location of this intersection while maintaining the same level of clean accuracy.
Therefore, the mixed classifier structure allows for optimizing the certified accuracy at a particular radius, while keeping the clean accuracy unchanged.
In \Cref{fig:certified_pareto}, we illustrate the achievable accuracy at each radius with the optimal $\alpha$-$\sigma^2$ combination as the \emph{Pareto Frontier}.
Compared with the accuracy-radius curve of a standalone RS classifier, this frontier significantly extends along the radius axis.
Since the clean accuracy is kept fixed in this comparison, a noticeable accuracy-robustness trade-off improvement can be concluded in the certified setting.

The RS-based bound from \cref{thm:randomized_smoothing} tightens the certification when the certifiably robust classifier is an RS model.
\Cref{fig:certified_radii} adds these tightened results to the visualizations.
For both $\halpha (\cdot)$ and $\hbase (\cdot)$, the RS-based bounds certify larger radii than the corresponding Lipschitz-based bounds.
Nonetheless, $\hbase(\cdot)$ can certify more points with the RS-based guarantee.
Intuitively, this phenomenon suggests that RS models can yield correct but low-confidence predictions when under attack with a large radius, and thus may not be best-suited for our mixing operation, which relies on robustness with non-zero margins.
In contrast, Lipschitz models, a more general and common class of models, exploit the mixing operation more effectively.
Moreover, as shown in \Cref{fig:STD+ROB}, empirically robust models often yield high-confidence predictions when under attack, making them more suitable for the mixed classifier $\halpha (\cdot)$'s robust base model.

Since randomized smoothing requires thousands of neural network queries to perform a prediction and the mixed classifier only adds one additional query via the standard base classifier, the change in computation is negligible.

\section{Additional Analyses Regarding \texorpdfstring{$R_i (x)$}{R\_i(x)}} \label{sec:more_R_analyses}

\subsection{The four options for \texorpdfstring{$R_i (x)$}{R\_i(x)}} \label{sec:R_options}

Consider the four listed options of $R_i (x)$, namely $1$, $\norm{\nabla g_i (x)}_{p*}$, $\norm{\nabla \max_j g_j (x)}_{p*}$, and $\frac{\norm{\nabla g_i (x)}_{p*}} {\norm{\nabla h_i (x)}_{p*}}$.
The constant $1$ is a straightforward option.
$\norm{\nabla g_i (x)}_{p*}$ comes from \cref{eq:adap_sm_2}, which is a direct generalization from the locally biased smoothing (binary classification) formulation to the multi-class case.
Note that $\norm{\nabla g_i (x)}_{p*}$ is not practical for datasets with a large number of classes, since it requires the calculation of the full Jacobian of $g (x)$, which is very time-consuming.
To this end, we use the gradient of the predicted class (which is intuitively the most important class) as a surrogate for all classes, bringing the formulation $\norm{\nabla \max_j g_j (x)}_{p*}$.
Finally, unlike locally biased smoothing, which only has one differentiable component, our adaptive smoothing has two differentiable base networks.
Hence, it makes sense to consider the gradient from both of them.
Intuitively, if $\norm{\nabla g_i (x)}_{p*}$ is large, then $g (\cdot)$ is vulnerable at $x$ and we should trust it less.
If $\norm{\nabla h_i (x)}_{p*}$, then $h (\cdot)$ is vulnerable and we should trust $h (\cdot)$ less.
This leads to the fourth option, which is $\frac{\norm{\nabla g_i (x)}_{p*}} {\norm{\nabla h_i (x)}_{p*}}$.

\subsection{Additional empirical supports and analyses for selecting \texorpdfstring{$R_i (x) = 1$}{R\_i(x)=1}} \label{sec:more_compare_R}

In this section, we use additional empirical evidence (Figures \ref{fig:compare_R_2a} and \ref{fig:compare_R_2b}) to show that $R_i (x) = 1$ is the appropriate choice for the adaptive smoothing formulation, and that the post-Softmax probabilities should be used for smoothing.
While most of the experiments in this paper are based on ResNets, the architecture is chosen solely because of its popularity, and our method does not depend on any properties of ResNets.
Therefore, for the experiment in \Cref{fig:compare_R_2a}, we select an alternative architecture by using a more modern ConvNeXt-T model \citep{Liu22} pre-trained on ImageNet-1k as $g (\cdot)$.
We also use a robust model trained via TRADES in place of an adversarially-trained network for $h (\cdot)$.
Moreover, while most of our experiments are based on $\ell_\infty$ attacks, our method applies to all $\ell_p$ attack budgets.
In \Cref{fig:compare_R_2b}, we provide an example that considers the $\ell_2$ attack. The experiment settings are summarized in \cref{tab:compare_R_settings}.

\begin{figure}[t]
	\centering
	\begin{subfigure}[t]{.48\textwidth}
		\captionsetup{justification=centering}
		\centering
		\includegraphics[width=.97\textwidth]{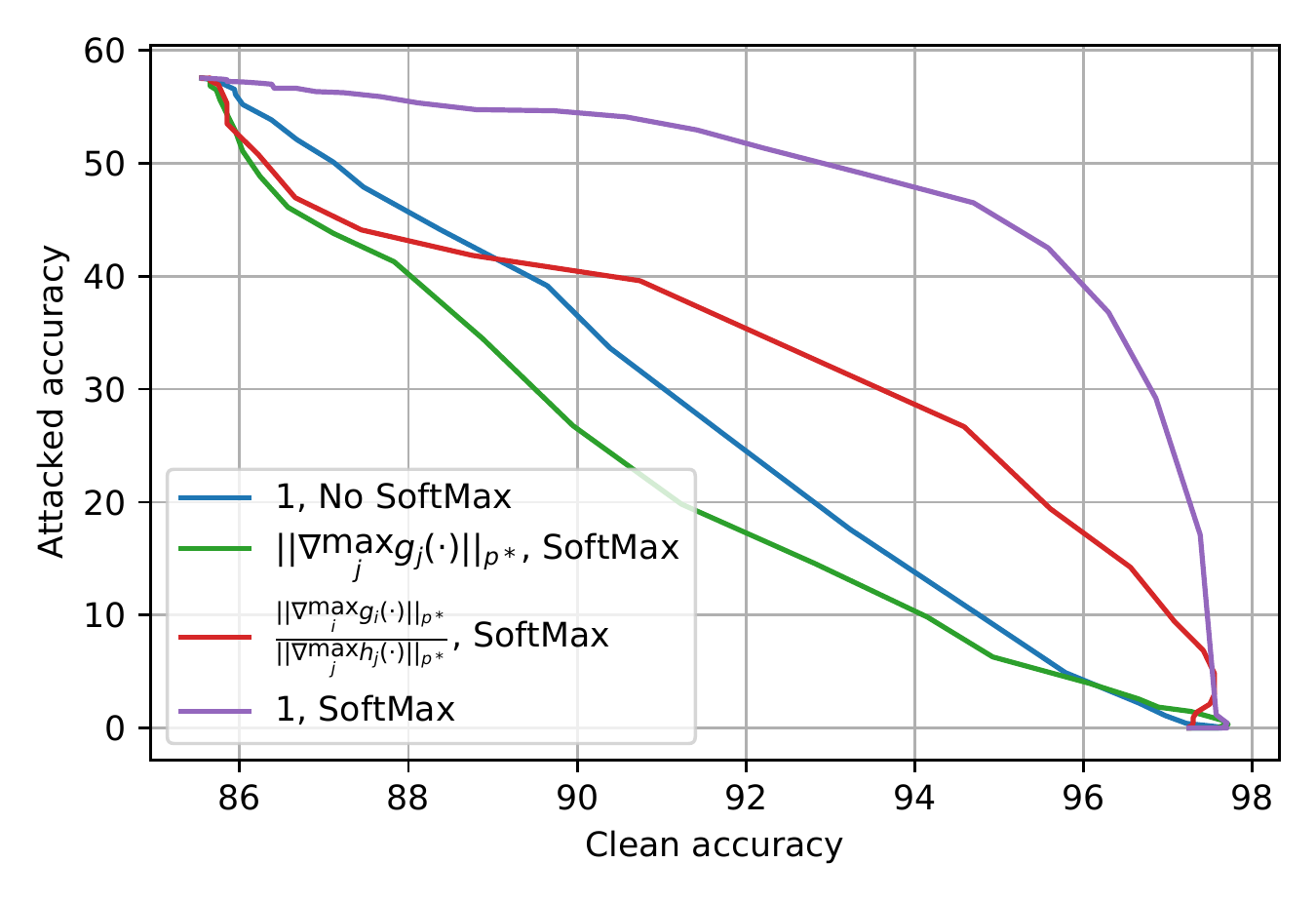}
		\vspace*{-2mm}
		\caption{\footnotesize ConvNeXt-T and TRADES WideResNet-34 under $\ell_\infty$ PGD attack.}
		\label{fig:compare_R_2a}
	\end{subfigure}
	\hspace{4mm}
	\begin{subfigure}[t]{.48\textwidth}
		\captionsetup{justification=centering}
		\centering
		\includegraphics[width=.97\textwidth]{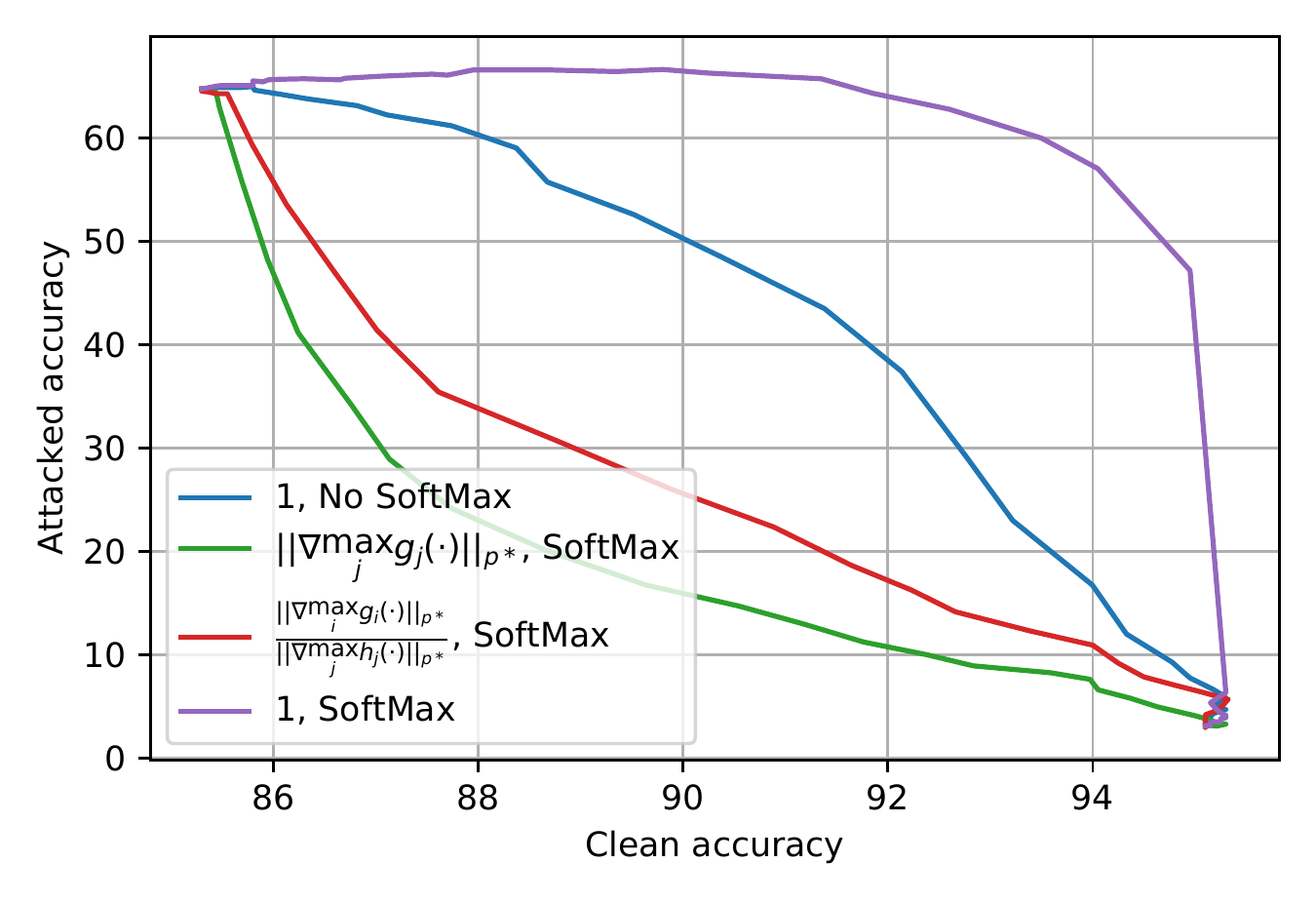}
		\vspace*{-2mm}
		\caption{\footnotesize Standard and AT ResNet-18s under $\ell_2$ PGD attack.}
		\label{fig:compare_R_2b}
	\end{subfigure}
	
	\vspace{-4.5mm}
	\caption{Comparing the ``attacked accuracy versus clean accuracy'' curve of various options for $R_i (x)$ with alternative selections of base classifiers.}
	\label{fig:compare_R_2}
\end{figure}

\begin{table}[t]
	\centering
	\vspace{-1mm}
	\caption{Experiment settings for comparing the choices of $R_i (x)$.}
	\vspace{-1mm}
	\label{tab:compare_R_settings}
	\begin{small}
	\begin{tabular}{l|c|c|c}
		\toprule
		& PGD attack settings	& $g (\cdot)$ Architecture	& $h (\cdot)$ Architecture \\
		\midrule
		\Cref{fig:compare_R}    & $\ell_\infty$, $\epsilon = \nicefrac{8}{255}$, 10 Steps	& Standard ResNet-18	& $\ell_\infty$ AT ResNet-18				\\
		\Cref{fig:compare_R_2a} & $\ell_\infty$, $\epsilon = \nicefrac{8}{255}$, 20 Steps	& Standard ConvNeXt-T & $\ell_\infty$ TRADES WideResNet-34		\\
		\Cref{fig:compare_R_2b} & $\ell_2$, $\epsilon = 0.5$, 20 Steps		  		  		& Standard ResNet-18	& $\ell_2$ AT ResNet-18					\\
		\bottomrule
	\end{tabular}
	\end{small}
	\vspace{1.5mm}
\end{table}

\Cref{fig:compare_R_2} demonstrates that setting $R_i (x)$ to the constant $1$ achieves the best trade-off curve between clean and attacked accuracy.
Moreover, smoothing using the post-Softmax probabilities outperforms the pre-Softmax logits.
This result aligns with the conclusions of \Cref{fig:compare_R} and our theoretical analyses, demonstrating that various robust networks share the property of being more confident when classifying correctly than when making mistakes.

The most likely reason for $R_i (x) = 1$ to be the best choice is that while the local Lipschitzness of a base classifier is a good estimator of its robustness and trustworthiness (as motivated in \citep{Anderson21b}), the gradient magnitude of this base classifier at the input is not always a good estimator of its local Lipschitzness.
Specifically, local Lipschitzness, as defined in \cref{def:lipschitz}, requires the classifier to be relatively flat within an $\epsilon$-ball around the input, whereas the gradient magnitude only focuses on the nominal input itself and does not consider the surrounding landscape within the $\epsilon$-ball.
For example, the gradient magnitude of the standard base classifier $g (\cdot)$ may jump from a small value at the input to a large value at some nearby point within the $\epsilon$-ball, which may cause $g (\cdot)$ to change its prediction around this nearby point.
In this case, $\norm{\nabla g(x))}$ may be small, but $g (\cdot)$ can have a high local Lipschitz constant.

As a result, while using $\norm{\nabla g(\cdot))}$ as $R_i$ seems to make sense at first glance, it does not work as intended and can make the mixed classifier trust $g (\cdot)$ more than it should.
Therefore, within the $\epsilon$-ball around a given $x$, the attacker may be able to find adversarial perturbations at which the gradient magnitude is small, thereby bypassing the defense.

In fact, as discussed in \citep{Anderson21b}, the use of gradient magnitude is motivated by approximating a neural classifier with a linear classifier.
Our \cref{fig:compare_R}, which demonstrates that using a constant $R_i (x)$ outperforms incorporating the gradient magnitude, implies that such an approximation results in a large mismatch and therefore does not make sense in our setting.
	
Even if some gradient-dependent options for $R_i (x)$ are better than the constant $1$, unless they produce significantly better results, the constant $1$ should still be favored since it removes the need for performing backward passes within the forward pass of the mixed classifier, making the mixing formulation more efficient and less likely to suffer from gradient masking.

\section{Additional Experiment Results}

\subsection{Trade-Off Curve with State-of-the-Art Base Classifiers} \label{sec:tradeoff_sota}

\begin{figure*}[]
	\centering
	\begin{minipage}{.565\textwidth}
	\centering
		\includegraphics[width=\textwidth]{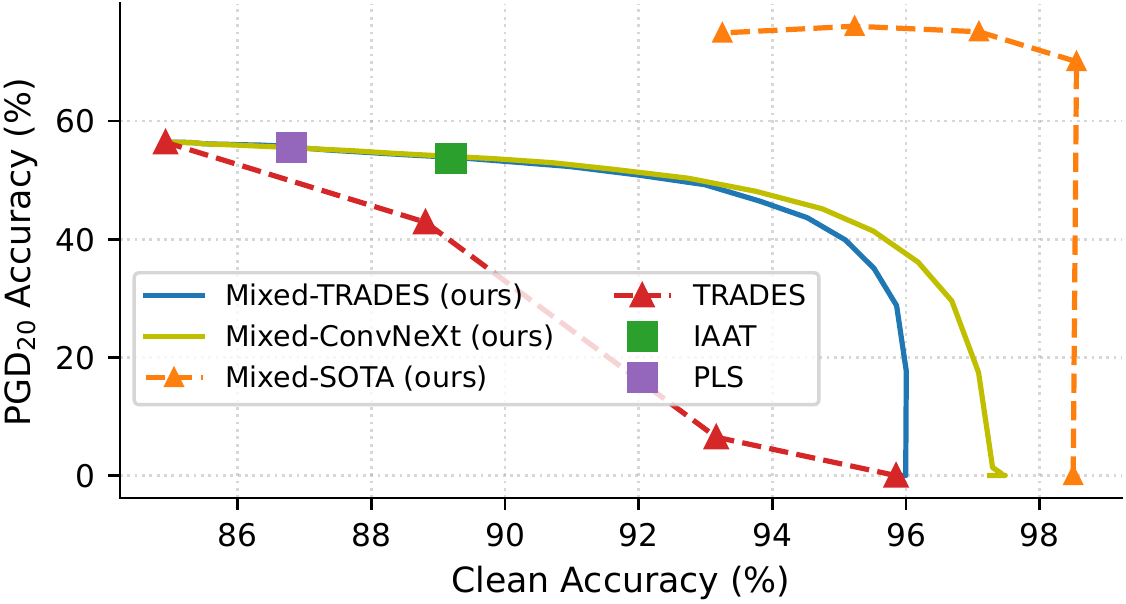}
	\end{minipage}
	\hspace{-8mm}
	\begin{minipage}{.47\textwidth}
	\begin{small}
	\begin{itemize}
		\item Mixed-TRADES: Adaptive smoothing using TRADES ($\beta = 6$) as $h (\cdot)$ and TRADES ($\beta = 0$) as $g (\cdot)$, as in \Cref{fig:compare_w_trades}.
		\item Mixed-ConvNeXt: Same as above but replace $g (\cdot)$ with a standard ConvNeXt-T model which has a similar size.
		\item Mixed-SOTA: The trade-off achieved by the SOTA mixed model in \cref{tab:compare_cifar}.
	\end{itemize}
	\end{small}
	\end{minipage}
	\caption{The mixed classifier's trade-off curve in \Cref{fig:compare_w_trades} can be easily improved by using a better base classifier. When using state-of-the-art models as base classifiers, adaptive smoothing achieves significantly better results than IAAT.}
	\label{fig:tradeoff_sota}
	\vspace{-3mm}
\end{figure*}

As discussed in the main paper text, our mixed classifier framework can take advantage of various models with better accuracy-robustness trade-offs (including but not limited to IAAT) by using them as base models, achieving state-of-the-art accuracy-robustness balance.

To demonstrate this, \Cref{fig:tradeoff_sota} adds the result that replaces the accurate base classifier used in \Cref{fig:compare_w_trades} with a ConvNeXt-T model, which has higher clean accuracy.
Such a replacement immediately improves the accuracy-robustness trade-off of the mixed classifier without additional training.
On the other hand, improving IAAT will at least require training a new model with expensive adversarial training.

Additionally, \Cref{fig:tradeoff_sota} displays the result achieved with state-of-the-art base classifiers from \cref{tab:compare_cifar}.
With these base classifiers, our mixed classifier can significantly improve the accuracy-robustness balance over training-based trade-off alleviating methods.
Since SOTA base classifiers use a variety of training techniques to achieve high performance, it is uncertain whether these techniques can be successfully combined with IAAT.
Meanwhile, incorporating them into adaptive smoothing is extremely straightforward.

\subsection{Ablation Study on Loss Function Hyperparameters} \label{sec:loss_abla}

In this section, we discuss the effects of the constants $\cCE$, $\cBCE$, and $\cprod$ in the composite loss function \cref{eq:comp_loss}.
Since multiplying the three weight constants by the same number is equivalent to using a larger optimizer step size and is not the focus of this ablation study (we focus on the loss function shape), we fix $\cCE + \cBCE = 1.5$.
To avoid the issue of becoming excessively conservative and always prioritizing the robust base model (as described in \Cref{sec:train_mixing_network}), we add a batch normalization layer without trainable affine transform to the output of the mixing network.
Additionally, note that since the mixing network has a single output, one can arbitrarily shift this output to achieve the desired balance between clean and attacked accuracies.
For a fair and illustrative comparison, after training a mixing network $\alpha_\theta (\cdot)$ with each hyperparameter setting, we add an appropriate constant to the output of the $\alpha_\theta (\cdot)$ so that the clean accuracy of the overall model $\halphatheta (\cdot)$ is $90 \pm 0.02 \%$, and compare the PGD$_{20}$ attacked accuracy of $\halphatheta (\cdot)$ in \cref{tab:loss_abla}.
As a baseline, when the smoothing strength $\alpha$ is a constant, the PGD$_{20}$ accuracy is $52.6 \%$ when the clean accuracy is tuned to be $90 \%$ (the corresponding $\alpha$ value is 1.763).
The above results demonstrate that $\cCE = 0$, $\cBCE = 1.5$, and $\cprod = 0.2$ works the best.

\begin{table}[!tb]
\centering
\vspace{-2mm}
\caption{The PGD$_{20}$ accuracy on CIFAR-10 with various loss hyperparameter settings. The setting is the same as in \cref{tab:cifar10}, and we consider both attack and defense in Setting B.}
\label{tab:loss_abla}
\vspace{-.8mm}
\begin{small}
\begin{tabular}{l!{\vrule width 2pt}c|c|c|c}
	\toprule
					& $\cCE = 0$	& $\cCE = 0.5$	& $\cCE = 1$		& $\cCE = 1.5$	\\
					& $\cBCE = 1.5$ & $\cBCE = 1$	& $\cBCE = 0.5$ & $\cBCE = 0$	\\
	\midrule
	$\cprod = 0$	& $54.5 \%$		& $52.8 \%$		& $53.8 \%$		& $54.4 \%$		\\
	$\cprod = 0.1$	& $54.3 \%$		& $54.1 \%$		& $54.0 \%$		& $54.1 \%$		\\
	$\cprod = 0.2$	& $55.1 \%$		& $54.2 \%$		& $54.3 \%$		& $53.9 \%$		\\
	\bottomrule
\end{tabular}
\end{small}
\end{table}

Our results also show that a small positive $\cprod$ is generally beneficial.
This makes sense because the CE loss is low for a particular input if both $g (\cdot)$ and $h (\cdot)$ correctly predict its class.
Thus, the smoothing strength should not matter for such input, and therefore the BCE loss is weighted by a small number.
Compared with using only the BCE loss, the product term of the CE and the BCE components is lenient on inputs correctly classified by the mixed model $\halphatheta (\cdot)$, while assigning more weight to the data that are incorrectly predicted.

\begin{table}
\centering
\vspace{-.5mm}
\caption{Ablation study on the mixing network's Sigmoid activation scaling factor.}
\label{tab:scale_abla}
\vspace{-.8mm}
\begin{small}
\begin{tabular}{l!{\vrule width 2pt}c|c|c|c}
	\toprule
	Scale  				& $0.5$		& $1$		& $2$		& $4$		\\
	\midrule
	PGD$_{20}$ Accuracy & $55.1 \%$	& $55.5 \%$	& $55.7 \%$	& $55.6 \%$	\\
	\bottomrule
\end{tabular}
\end{small}
\end{table}

Recall that the output range of $\alpha_\theta (\cdot)$ is $[0, 1]$, which is enforced by appending a Sigmoid output activation function.
In addition to shifting, one can arbitrarily scale the Sigmoid activation's input.
By performing this scaling, we effectively calibrate the confidence of the mixing network.
In \cref{tab:loss_abla}, this scaling is set to the same constant for all settings.
In \cref{tab:scale_abla}, we select the best loss parameter and analyze the validation-time Sigmoid scaling.
Again, we shift the Sigmoid input so that the clean accuracy is $90 \pm 0.02 \%$.
While a larger scale benefits the performance on clean/attacked examples that are confidently recognized by the mixing network, an excessively large scale makes $\halphatheta (\cdot)$ less stable under attack.
\cref{tab:scale_abla} shows that applying a scaling factor of $2$ yields the best result for the given experiment setting.

\subsection{Estimating the Local Lipschitz Constant for Practical Neural Networks} \label{sec:est_lip}

In this section, we demonstrate the practicality of \cref{thm:certified_radius} by showing that it can work with a relaxed local Lipschitz counterpart of \cref{as:lipschitz}, which can be estimated for practical differentiable models.

First, note that the proof \cref{thm:certified_radius} does not require global Lipschitzness, and local Lipschitzness will suffice.
Since the local Lipschitz constant of an empirically robust (AT, TRADES, etc.) neural classifier can be much smaller than the global Lipschitz constant,  \cref{thm:certified_radius} is less restrictive in practice.
Moreover, it is not necessary for the model output to be similar between an arbitrary pair of inputs within the $\epsilon$ ball.
Instead, \cref{thm:certified_radius} only requires the model output to not change too much with respect to the nominal unperturbed input.
Furthermore, \cref{thm:certified_radius} only requires single-sided Lipschitzness.
Namely, we only need to make sure that the predicted class probability does not decrease too much compared with the nominal input, and whether this probably becomes even higher than the nominal input will not affect robustness.
The opposite is true for the non-predicted classes.

Specifically, suppose that for an arbitrary input $x$ and an $\ell_p$ attack radius $\epsilon$, the following two conditions hold with respect to the local Lipschitz constant $\lip_p^x$:
\begin{itemize}[leftmargin=6mm]
	\item $\sigma \circ h_y (x) - \sigma \circ h_y (x + \delta) \le \epsilon \cdot \lip_p^x (\sigma \circ h_y)$ and $\sigma \circ h_i (x + \delta) - \sigma \circ h_i (x) \le \epsilon \cdot \lip_p^x (\sigma \circ h_i)$ for all $i \neq y$ and all perturbations $\delta$ such that $\norm{\delta}_p \leq \epsilon$;
	\item The robust radius $\rlippalpha (x)$ as defined in \cref{eq:lip_cert_rad} but use the local Lipschitz constant $\lip_p^x$ as a surrogate to the global constant $\lip_p$, is not smaller than $\epsilon$.
\end{itemize}
Then, if the robust base classifier is correct at the nominal point $x$, then the mixed classifier is robust at $x$ within the radius $\epsilon$.
The proof of this statement follows the proof of \Cref{thm:certified_radius}.

Moreover, the literature \citep[Eq.(3)]{Yang20} has shown that the local Lipschitz constant of a given differentiable classifier can be easily estimated using a PGD-like algorithm.
The work \citep{Yang20} also showed that many existing empirically robust models, including those trained with AT or TRADES, are in fact locally Lipschitz.
Note that \citep{Yang20} evaluates the local Lipschitz constants of the logits, whereas we analyze the probabilities, whose Lipschitz constants are much smaller.

Here, we modify the PGD-based local Lipschitzness estimation for the relaxed requirement of \cref{thm:certified_radius} listed above.
Specifically, we estimate the local Lipschitz constant within an $\epsilon$-ball around an arbitrary input $x$ by using the PGD algorithm to solve the problem
\begin{equation} \label{eq:est_lipschitz}
	\widehat{\lip_p^x} (\sigma \circ h) \coloneqq \frac{1}{c \cdot \epsilon} 
	\bigg( \max_{\norm{\delta}_p \leq \epsilon} \big( \sigma \circ h_y (x + \delta) - \sigma \circ h_y (x) \big) - 
	\sum_{i \neq y} \max_{\norm{\delta}_p \leq \epsilon} \big( \sigma \circ h_i (x + \delta) - \sigma \circ h_i (x) \big) \bigg),
	\\ \vspace{-2mm}
\end{equation}
where $\widehat{\lip_p^x} (\sigma \circ h)$ is the estimated local Lipschitzness of $\sigma \circ h (\cdot)$ averaged among all classes, and $c$ is the number of classes as defined in \Cref{sec:notations}.
Unlike in \citep{Yang20}, we decouple the classes by maximizing each class's probability deviation separately, providing a more conservative and insightful estimation.

We use the default TRADES WideResNet-34-10 model as an example to demonstrate robust neural networks' non-trivial Lipschitzness.
When using the PGD$_{20}$ algorithm to solve \cref{eq:est_lipschitz}, the estimated Lipschitz constant $\widehat{\lip_p^x} (\sigma \circ h)$ is 3.986 averaged among all test data within the $\ell_\infty$ ball with radius $\frac{8}{255}$.
Note that this number is normalized with $\epsilon$, which is a small number.
Intuitively, this Lipschitz constant implies that the probability of a class changes for merely at most $0.125$ within this $\ell_\infty$ attack budget on average.
Therefore, the local Lipschitz constant, which is what \Cref{thm:certified_radius} relies on, is not large for robust deep neural networks.
	
Since the relaxed Lipschitz constant can be estimated for differentiable classifiers and is not excessively large for robust models, the certified bound is not small.
Hence, \cref{thm:certified_radius} provides important theoretical insights into the empirical robustness of the mixed classifier.

\section{Experiment Implementation Details}

\subsection{Implementation of the Mixing Network in Experiments} \label{sec:mixing_arch_rn}

Since our formulation is agnostic to base classifier architectures, \Cref{fig:mixing_arch} in the main text presents the design of the mixing network in the context of general standard and robust classifiers.
In the experiments presented in \Cref{sec:ada_exp}, both $g (\cdot)$ and $h (\cdot)$ are based on ResNet variants, which share the general structure of four main blocks, resulting in \Cref{fig:mixing_arch_rn} as the overall structure of the mixed classifier.
Following \citep{Metzen17}, we consider the initial Conv2D layer and the first ResNet block as the upstream layers.
The embeddings extracted by the first Conv2D layers in $g (\cdot)$ and $h (\cdot)$ are concatenated before being provided to the mixing network $\alpha_\theta (\cdot)$.
We further select the second ResNet block as the middle layers.
For this layer, in addition to concatenating the embeddings from $g (\cdot)$ and $h (\cdot)$, we also attach a linear transformation layer (Conv1x1) to match the dimensions, reduce the number of features, and improve efficiency.

\begin{figure}[!tb]
    \centering
    \includegraphics[width=\textwidth]{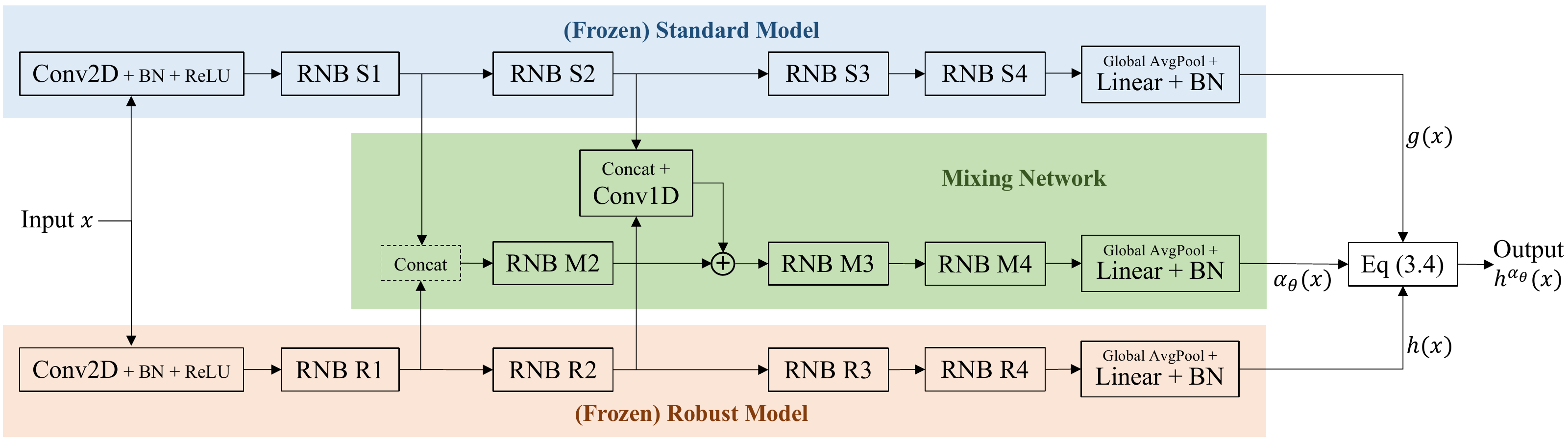} \vspace{-5mm}
    \caption{The architecture of the mixed classifier introduced in \Cref{sec:ada_smo} when applied to a pair of ResNet base models.}
    \label{fig:mixing_arch_rn}
    \vspace{-2.5mm}
\end{figure}

As mentioned in \Cref{sec:alpha(x)}, the range of $\alpha_\theta (\cdot)$ can be constrained to be within $(\alphamin, \alphamax) \subseteq [0, 1]$ if certified robustness is desired.
We empirically observe that setting $\alphamax - \alphamin$ to be $0.04$ works well for CIFAR-10, whereas $0.1$ or $0.15$ works well for CIFAR-100.
This observation coincides with \Cref{fig:STD+ROB}, which shows that a slight increase in $\alpha$ can greatly enhance the robustness at the most sensitive region.
The value of $\alphamin$ can then be determined by enforcing a desired level of either clean validation accuracy or robustness.
Following this guideline, we set the ranges of $\alpha_\theta (\cdot)$ to be $(0.96, 1)$ for the CIFAR-10 model discussed in \cref{tab:compare_cifar}. The range is $(0.84, 0.99)$ and $(0.815, 0.915)$ respectively for the two CIFAR-100 models in \cref{tab:compare_cifar}.
Note that this range is only applied during validation. When training $\alpha_\theta (\cdot)$, we use the full $(0, 1)$ range for its output, so that the training-time adversary can generate strong and diverse attacks that fully exploit $\alpha_\theta (\cdot)$, which is crucial for securing the robustness of the mixing network.
We also observe that exponential moving average (EMA) improves the training stability of the mixing network, and applies an EMA decay rate of 0.8 for the model in \cref{tab:compare_cifar}.
Furthermore, scaling the outputs of $h (\cdot)$ by a number between 0 and 1 and scaling the outputs of $g (\cdot)$ by a number greater than 1 can help with the overall accuracy-robustness trade-off.
This scale is set to 3 for the experiments in \cref{tab:compare_cifar}.

\subsection{AutoAttack for Calculating the Robust Confidence Gap} \label{sec:autoattack_margin}

The original AutoAttack implementation released with \citep{Croce20a} does not return perturbations that fail to change the model prediction.
To enable robustness margin calculation, we modify the code to always return the perturbation that achieves the smallest margin during the attack optimization.
Since the FAB and Square components of AutoAttack are slow and do not successfully attack additional images on top of the APGD components, we only consider the two APGD components for the purpose of margin calculation.

\end{document}